\documentclass[10pt,twocolumn,letterpaper]{article}

\usepackage[cvprfinal]{cvpr}      %

\definecolor{cvprblue}{rgb}{0.21,0.49,0.74}
\usepackage[pagebackref,breaklinks,colorlinks,allcolors=cvprblue]{hyperref}
\usepackage{booktabs}       %
\usepackage{amsfonts}       %
\usepackage{nicefrac}       %
\usepackage{microtype}      %
\usepackage{xcolor}         %
\usepackage{amsmath}        %
\usepackage{makecell}
\usepackage{amsthm}
\usepackage{algpseudocode}
\allowdisplaybreaks

\def\paperID{43374} %
\def\confName{CVPR}
\def\confYear{2026}

\title{Efficient Equivariant Transformer for Self-Driving Agent Modeling}

\author{
\textbf{Scott Xu$^{1,2}$, Dian Chen$^{\dagger}$, Kelvin Wong$^{1,2}$, Chris Zhang$^{1,2}$, Kion Fallah$^{\dagger}$, Raquel Urtasun$^{1,2}$} \\
Waabi$^1$, University of Toronto$^2$ \\
\texttt{\{sxu, kwong, czhang, urtasun\}@waabi.ai}
\vspace{-0.5em}
}

\newcommand{\refsec}[1]{Section \ref{#1}}
\newcommand{\reffig}[1]{Figure \ref{#1}}
\newcommand{\reftbl}[1]{Table \ref{#1}}
\newcommand{\refeq}[1]{Equation \ref{#1}}

\newtheorem{proposition}{Proposition}
\newtheorem{lemma}{Lemma}
\newtheorem*{remark}{Remark}
\algnewcommand{\LComment}[1]{\State \textit{// #1}}

\makeatletter
\DeclareRobustCommand\onedot{\futurelet\@let@token\@onedot}
\def\@onedot{\ifx\@let@token.\else.\null\fi\xspace}
\def\eg{\emph{e.g}\onedot} 
\def\ie{\emph{i.e}\onedot} 
 
\def\etc{\emph{etc}\onedot} \def\vs{\emph{vs}\onedot}

\makeatother

\begin{document}

\def\GA{{\mathbb{R}^*_{2,0,1}}}
\def\SE2{{\text{SE}(2)}}
\def\ourmodel{{\text{DriveGATr}}\xspace}

\maketitle

\begingroup
\renewcommand{\thefootnote}{}
\footnotetext{$^\dagger$ Work done while at Waabi.}
\endgroup

\begin{abstract}

    Accurately modeling agent behaviors is an important task in self-driving.
    It is also a task with many symmetries, such as equivariance to the order of
    agents and objects in the scene or equivariance to arbitrary roto-translations
    of the entire scene as a whole; \ie, $\SE2$-equivariance.
    The transformer architecture is a ubiquitous tool for modeling these symmetries.
    While standard self-attention is inherently permutation equivariant,
    explicit pairwise relative positional encodings have been the standard for introducing $\SE2$-equivariance.
    However, this approach introduces an additional cost that is quadratic in the number of agents,
    limiting its scalability to larger scenes and batch sizes.
    In this work, we propose DriveGATr, a novel transformer-based architecture for agent modeling that achieves
    $\SE2$-equivariance without the computational cost of existing methods.
    Inspired by recent advances in geometric deep learning, DriveGATr encodes scene elements
    as multivectors in the 2D projective geometric algebra $\GA$ and processes
    them with a stack of equivariant transformer blocks.
    Crucially, DriveGATr models geometric relationships using standard attention
    between multivectors, eliminating the need for costly explicit pairwise relative positional encodings.
    Experiments on the Waymo Open Motion Dataset demonstrate that DriveGATr is comparable to the
    state-of-the-art in traffic simulation and establishes a superior Pareto front for performance
    \vs computational cost.
\end{abstract}

\section{Introduction}

Understanding how traffic agents behave has many important applications in self-driving,
from online motion forecasting for autonomy to offline traffic modeling for simulation.
In these applications, a traffic scene typically consists of a set of geometric objects
representing traffic agents, like vehicles and pedestrians, and contextual elements,
like the map lane graph and traffic signals~\citep{vectornet,lanegcn,wayformer}.
Our goal is to learn an agent model that accurately predicts the actions of each agent
in the scene from this context.
This task has many symmetries, chief among which is equivariance to arbitrary 
2D roto-translations of the scene; \ie, $\SE2$-equivariance.
That is, if we apply a rigid transformation $\mathcal{T}\in \SE2$ to the input scene,
the outputs of the model should also transform by that same rigid transformation,
preserving their relative geometric relationships.

A natural question arises: how can we imbue our agent models with this equivariance?
A common approach is to rely on data diversity.
The hope is that with enough training examples, the model naturally learns a roughly equivariant function.
While conceptually simple, this approach can be sample and compute inefficient
and yet still struggle to generalize to new scenes as it is not truly equivariant.
Alternatively, one can directly encode this property as an inductive bias into the model,
designing the network architecture such that it is equivariant by construction.
Approaches in this category typically
explicitly model pairwise relationships between agents. For instance, one can
transform the context around each agent into its own respective coordinate frame
and process each agent's context independently~\citep{motionlm,gigaflow} or use
pairwise relative positional encoding to process all agents simultaneously~\citep{gorela,behaviorgpt}.
Unfortunately, all of these approaches
introduce an additional cost that scales quadratically with the number of agents,
limiting its suitability to scale to ever larger scenes, models, datasets, \etc.

In this paper, we propose \ourmodel, a novel architecture for modeling agents
that achieves $\SE2$-equivariance without the computational cost of existing methods.
Our approach builds on recent advances in geometric deep learning.
In particular, we develop a novel extension of the Geometric Algebra Transformer (GATr)~\citep{gatr} to $\SE2$-equivariance
and new architectural primitives that improve its efficacy for agent modeling.
First, we propose an efficient encoding of scene elements into the 2D projective geometric algebra
$\GA$, which allows for a native representation of 2D objects (\eg, points, lines)
and operators (\eg, translations, rotations) as 8-dimensional multivectors.
Next, we develop $\SE2$-equivariant neural network primitives, such as linear / bilinear layers, nonlinearities, normalization and scaled dot product attention.
In addition, we design a new primitive to transform equivariant geometric representations into invariant features, since an agent's actions are ultimately invariant.
These primitives form the basis of \ourmodel's transformer-based architecture.

Our approach has two key advantages.
First, compared to non-equivariant models, \ourmodel is equipped with a geometric inductive bias
that improves its expressivity, sample efficiency, and robustness to nuissance transformations.
Second, unlike prior equivariant methods, \ourmodel achieves this without explicit pairwise relative positional encodings.
Instead, it models geometric relationships between agents and scene elements using standard
scaled dot product attention between multivectors, which have been heavily optimized~\citep{flashattention}.
This gives us flexibility to improve the model's performance by growing each agent's context
(\eg, to all agents or the entire lane graph) or scale to larger scenes, models, and datasets efficiently.
Our experiments on the Waymo Open Motion Dataset~\citep{womd} confirm these attributes, showing that
\ourmodel achieves comparable results to the state-of-the-art in traffic simulation while
establishing a superior Pareto front in performance \vs computation cost compared to the baselines.

To summarize, we introduce \ourmodel, a tailoring of GATr to acheive efficient
$\SE2$-equivariance with internal geometric algebra representations.
We provide theoretical analysis showing it is
provably equivariant by construction. We develop novel components targeted 
for agent modeling in self-driving, and empirically validate the effectiveness
through extensive large scale experiments.

\section{Related work}
\subsection{\SE2 equivariance in traffic modeling}
A natural symmetry in traffic modeling is equivariance to 2D roto-translations of the scene;
\ie, \SE2 equivariance.
This symmetry gives rise to a powerful inductive bias for improving our traffic models'
expressivity, efficiency, and generalization~\cite{waymo-coordinate-frame}.
Indeed, equivariant models tend to outperform their non-equivariant baselines at any compute or data budget~\cite{does_equivariance_matter}. 
While the advantage of this inductive bias diminishes at scale,
the scale required far exceeds that of any public self-driving datasets today,
where equivariant models still top the leaderboards~\cite{smart,catk}.
Moreover, at a fixed compute budget, the optimal equivariant models are smaller,
providing an additional inference time advantage.

To achieve \SE2 equivariance, existing models explicitly compute relative poses between
each pair of actors and map elements in the scene~\cite{gorela,smart,catk,behaviorgpt}.
In particular, state-of-the-art models like SMART~\cite{smart,catk} use a transformer-based
architecture that extends the attention dot product with Relative Positional Embeddings (RPE),
\begin{equation}
    \alpha_{ij} = \frac{\langle Q_i, K_j + \operatorname{RPE}(\text{pos}_{i \to j}) \rangle}{\sqrt{d_k}}
\end{equation}
where $ \operatorname{RPE}(\text{pos}_{i \to j}) $ featurizes the relative pose from $ i$ to $ j $.
Explicitly computing pairwise RPEs scales quadratically with the number of scene elements,
limiting the suitability of these models to scale to larger scenes, models, and datasets.
Moreover, because they do not use standard scaled dot product attention, these models cannot
use highly optimized attention kernels like FlashAttention~\cite{flashattention}.

To address this limitation, recent work have explored ways to avoid explicit RPEs.
For example, DRoPE~\citep{drope} extends 1D rotary positional embeddings (RoPE)~\citep{roformer} to 2D,
modifying the attention dot product so that query and key embeddings are first transformed by
blockwise rotation matrices $ \mathbf{R}_i $ and $ \mathbf{R}_j $ encoding poses of $ i $ and $ j $
\begin{equation}
    \alpha_{ij} = \frac{\langle \mathbf{R}_i Q_i, \mathbf{R}_j K_j \rangle}{\sqrt{d_k}}
\end{equation}

While DRoPE bypasses the scalability issue, it lacks expressivity as it does not encode explicit geometric information about the scene elements.
It is also translation-equivariant but not rotation-equivariant.
Alternatively, VN-Transformer~\citep{vn_transformer} encodes poses as Vector-Neurons~\citep{vector_neurons}
and processes them with a transformer with SO(3) equivariant layers.
However, for numerical stability, VN-Transformer requires modifications that sacrifice true equivariance.

\subsection{Geometric Algebra Transformer}
Recently, a novel method from geometric deep learning, called the Geometric Algebra Transformer (GATr)~\cite{gatr}, proposes an E(3)-equivariant transformer-based architecture which represents geometric objects as elements of a projective geometric algebra.
This projective geometric algebra provides a single, unified 16-dimensional algebraic structure to represent all types of geometric data, including points, lines, planes, rotations, and translations.
GATr consists of E(3)-equivariant network primitives, including equivariant attention, MLP, and normalization layers, which operate directly on these geometric algebra elements.
Because it is built on the standard transformer framework, GATr is scalable, versatile, and efficient.

However, applying GATr directly to traffic agent modeling not only results in redundancy, but also inaccuracy.
Notably, the original GATr architecture is $\text{E}(3)$-equivariant, whereas a driving scene is only \SE2-equivariant due to gravity and right-hand traffic rules.
In this work, we derive and implement efficient 2D geometric encodings and equivariant layers which operate on an algebraic structure using 8 dimensions instead of 16.

\section{\ourmodel}
\label{sec:methods}

\begin{figure*}[t]
    \centering
    \begin{subfigure}[h]{\textwidth}
        \includegraphics[trim={1cm 0 0 0}, width=\textwidth]{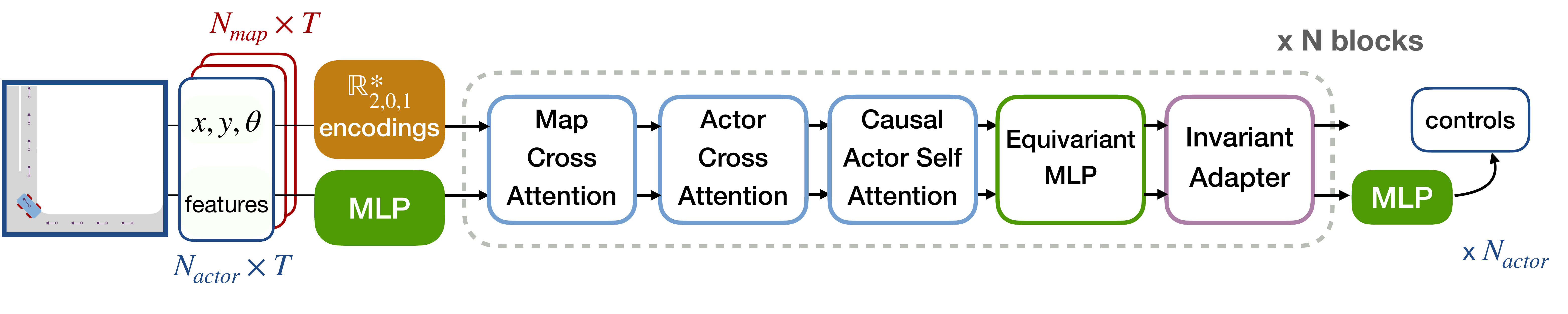}
        \vspace*{-0.6cm}
        \caption{Overview of the architecture.}\
        \label{pipeline_overview}
    \end{subfigure}
    \begin{subfigure}[b]{0.32\textwidth}
        \includegraphics[trim={7.5cm -3.5cm 60cm 0},page=2,width=\textwidth]{assets/overview.pdf}
        \caption{Multivector attention block.}
        \label{attention_overview}
    \end{subfigure}
    \begin{subfigure}[b]{0.57\textwidth}
        \includegraphics[trim={-9cm -3.5cm 60cm 0},page=3,width=0.8\textwidth]{assets/overview.pdf}
        \caption{Equivariant MLP block.}
        \label{mlp_overview}
    \end{subfigure}
    \caption{The \ourmodel architecture. \reffig{pipeline_overview} provides an overview.
        The poses and features of $N_{actor}$ agents and $N_{map}$ nodes in each scene are encoded as multivectors in $\GA$ and scalars.
        These tensors are processed by N transformer blocks, each consisting of agent and map cross attention, temporal causal self-attention, equivariant MLPs and invariant adapters.
        Each of these modules contain skip connections, closely mimicking a standard transformer.
        \reffig{attention_overview} describes the attention block.
        For all attention blocks, the query inputs are agent of interests.
        The key and value inputs are agents and map nodes for the cross attentions, attending across elements per-timestep.
        For self-attention, attention spans the temporal axis per-agent.
        \reffig{mlp_overview} describes the equivariant MLP block.
        \reftbl{tbl:encodings} provides details on $\GA$ encodings.
        \refsec{sec:equivariant_layers} provides details on each of the equivariant primitive layers.
        \refsec{sec:invariant_adapter} describes the invariant adapter block.
    }
    \label{overview}
\end{figure*}

This work is motivated by the goal of designing an efficient and expressive approach to equivariant traffic modelling.
As we will see later, a key idea in our proposed method will be to extend the attention dot product with more invariant dot products.
For example, if $h_{Q_i}$ and $h_{K_j}$ are 2D heading vectors, one could compute SE(2)-invariant attention logits via
\begin{equation}
    \alpha_{ij}=\frac{\langle Q_i,K_j\rangle+\langle h_{Q_i},h_{K_j}\rangle}{\sqrt{d_k+2}}
\end{equation}
so that query tokens attend more strongly to key tokens with similar heading.
However, this naive example has several limitations:
(i) $h_{Q_i}$ and $h_{K_j}$ are not learnable features, so this approach is not expressive;
(ii) this only reasons about headings, and not positions; and
(iii) in this example, the token features do not contain any explicit pose information.
Using geometric algebra encodings will allow us to address all of these limitations.

The geometric algebra $\GA$ gives us a unified and efficient way to represent 2D geometries, including points, lines, translations, and rotations.
We will show how to encode 2D poses as elements of this algebra, called \textit{multivectors}; in particular, the multivector representing the pose $(x,y,\theta)$ will have components corresponding to $x,y,\cos\theta$, and $\sin\theta$.
To make our geometric features learnable, we will derive \SE2-equivariant layers $\mathcal{L}:\GA\to\GA$, including linear layers, normalization, and nonlinearities.
The geometric algebra is also equipped with an invariant inner product $\langle\cdot,\cdot\rangle:\GA^2\to\mathbb{R}$, which allows us to compute SE(2)-invariant attention logits via
\begin{equation}
    \alpha_{ij}=\frac{\langle Q_i,K_j\rangle+\langle Q^{MV}_i,K^{MV}_j\rangle}{\sqrt{d_k+d^{MV}}} \label{eq:attn}
\end{equation}
where $Q^{MV}_i$ and $K^{MV}_j$ are multivector geometric features and $d^{MV}$ is the dimension of the invariant inner product.
Notably, ~\refeq{eq:attn} can be computed as standard dot-product attention via efficient kernels by concatenating $Q, K$ with (certain components of) $Q^{MV}, K^{KV}$.

In the next sections, we first provide an overview of \ourmodel.
We then provide a brief preliminary on generic geometric algebra and describe the details of \ourmodel, which is an efficient 2D extension of GATr~\citep{gatr}
with modifications for agent modeling.
We refer the readers to the supplementary material for additional preliminaries.

\subsection{DriveGATr Architecture}

Given a history of agent states $\mathcal{A}_t=\{a_1^{1:t}, \dots, a_N^{1:t} \}$ and map nodes $\mathcal{M}=\{m_1, \dots, m_M\}$, \ourmodel predicts an action for each agent: $\mu = \text{\ourmodel}(\mathcal{A}, \mathcal{M})$ at the current timestep.
A dynamics model $f$ advances the agent states: $a_i^{t+1} = f(a_i^t, \mu_i)$.
We repeat this process to unroll the traffic scene in $T$ timesteps.
By construction, the trajectories unrolled by the dynamics model are equivariant relative to the scene.
\reffig{overview} provides an overview of the architecture.

For every token in $\mathcal{A}$ and $\mathcal{M}$, we encode its 2D global\footnote{Theoretically, the choice of coordinate frame can be arbitrary due to SE(2) equivariance, but for numerical stability we use the ego's initial pose.} position $(x,y)$ with heading $\theta$ into a single multivector, using the bivector components to encode the point $(x,y)$, and the vector components to encode the line passing through $(x,y)$ with direction angle $\theta$.

We also encode invariant features in auxiliary scalars using a MLP.
For agent states, we encode their speed and bounding box dimensions in the auxiliary scalars.
For map segments, we encode their length, width, curvature, speed limit, and boundary types.

Our transformer consists of a sequence of factorized attention blocks, where each block consists of
(i) multi-head cross attention between the agent states and map tokens, per timestep;
(ii) multi-head self attention between the agent states, per timestep;
(iii) multi-head causal self-attention between agent states, per agent;
(iv) an equivariant MLP; and
(v) an invariant adapter described in Section~\ref{sec:invariant_adapter}.

We use a final MLP on the auxiliary scalars to decode invariant action logits for each agent.
We find it beneficial to attend agent states to the entire map in our cross-attention layer, rather than the nearest few map tokens.

\subsection{Geometric algebra}

\paragraph{Geometric product.} The geometric product $xy$ of $x, y \in \mathbb{R}^n$ is a way of multiplying vectors together.
The geometric product is defined to be bilinear and associative, and is characterized by the fundamental equation $v^2=\langle v,v\rangle$ where $\langle\cdot,\cdot\rangle$ is the standard inner product.
This induces an algebra, called a \textit{Euclidean geometric algebra}, whose elements are called \textit{multivectors}.
From the fundamental equation one can deduce that the geometric product is anticommutative on orthogonal basis vectors,
i.e. $e_ie_j=-e_je_i$ for $i \neq j$.
Consequently, for $n=2$ the Euclidean geometric algebra consists of multivectors of the form $x=x'+x_1e_1+x_2e_2+x_{12}e_{12}$, where $x',x_1,x_2,x_{12}$ are real coefficients and $e_{12}:=e_1e_2$.

This structure is useful because it provides a unified framework for representing and manipulating geometric objects.
For example, in 2D, we can rotate a vector $v=(x,y)$ by an angle $\theta$ using the geometric product:
\begin{align}
    R_{\theta}v&=(\cos\theta-(\sin\theta)e_{12})(xe_1+ye_2)\\
    &=(x\cos\theta-y\sin\theta)e_1+(x\sin\theta+y\cos\theta)e_2
\end{align}
where we have used that $e_1^2=e_2^2=1$ and $e_2e_1=-e_1e_2$.

\paragraph{Projective geometric algebra.} Euclidean geometric algebra is only able to represent linear transformations.
A translation is not a linear operation, but we can solve this by adding an extra, special dimension, similar to how homogeneous coordinates are used in computer graphics.

A projective geometric algebra is obtained by first adjoining an orthogonal basis vector $e_0$ to $\mathbb{R}^n$ with norm zero, that is, $e_0^2=0$.
For $n=2$, the resulting algebra, denoted by $\GA$, is 8-dimensional, spanned by the scalar 1, three vectors $e_0,e_1,e_2$, three \textit{bivectors} $e_{01},e_{20},e_{12}$, and one \textit{pseudoscalar} $e_{012}$.

\paragraph{Geometric algebra primitives.}
We define additional geometric algebra primitives that will be useful later.

The \emph{wedge product} is another way of multiplying vectors together, defined by $x\wedge y=(xy-yx)/2$.
It is bilinear and associative, and has the property that $v\wedge v=0$ for any vector $v$.
Also, for distinct basis vectors $e_i,e_j$, from $e_je_i=-e_ie_j$ one can derive that $e_i\wedge e_j=e_ie_j$ concides with the geometric product.
The wedge product also extends to general multivectors.

The \emph{multivector dual} $x\mapsto x^*$ is a linear operator, defined for a basis multivector $x$ as the multivector such that $x\wedge x^*$ equals the pseudoscalar.

The \emph{join} operator of two multivectors is defined as $\textrm{Join}(x,y)=(x^*\wedge y^*)^*$.

The $k$-\emph{blade projection} $\langle x \rangle_k$ is defined to select specific components of a multivector $x\in\GA$
\begin{align}
    \langle x\rangle_0&:=x'\\
    \langle x\rangle_1&:=x_0e_0+x_1e_1+x_2e_2\\
    \langle x\rangle_2&:=x_{01}e_{01}+x_{20}e_{20}+x_{12}e_{12}\\
    \langle x\rangle_3&:=x_{012}e_{012}
\end{align}

We define the \emph{invariant inner product} $\langle\cdot,\cdot\rangle$ between two multivectors in $\GA$ as
\begin{equation}
    \langle x,y\rangle:=x'y'+x_1y_1+x_2y_2+x_{12}y_{12}\in\mathbb{R}
\end{equation}
which ignores components containing $e_0$.

\subsection{Encodings}

\begin{table*}[t]
    \centering
    \begin{tabular}{ccc}
        \toprule
        Geometry / transform                             & Representation                                    & Transform inverse                                  \\
        \cmidrule(lr){1-1}      \cmidrule(lr){2-2}  \cmidrule(lr){3-3}
        Line $ax+by+c=0$                            & $ae_1+be_2+ce_0$                                  & -                                                 \\
        \rule{0pt}{3ex} 
        Point $(x,y)$                               & $xe_{20}+ye_{01}+e_{12}$                          & -                                                 \\
        \rule{0pt}{3ex} 
        Translation by $(a,b)$                      & $1-\frac{a}{2}e_{01}+\frac{b}{2}e_{20}$           & $1+\frac{a}{2}e_{01}-\frac{b}{2}e_{20}$           \\
        \rule{0pt}{3ex} 
        Counterclockwise rotation of angle $\theta$ & $\cos\frac{\theta}{2}-\sin\frac{\theta}{2}e_{12}$ & $\cos\frac{\theta}{2}+\sin\frac{\theta}{2}e_{12}$ \\
        \bottomrule
    \end{tabular}
    \caption{2D encodings of geometries and transforms.}
    \label{tbl:encodings}
    \vspace{-1em}
\end{table*}

We encode 2D objects and transformations into $\GA$ using~\reftbl{tbl:encodings} so that geometric features can be computed using the operations defined above.
Following~\citet{gatr}, we encode 2D geometries and transforms such that the result of applying a transform to a geometry is given by a sandwich product.
For example, if $t$ is a multivector representing a translation by $(a,b)$ and $p$ is a multivector representing the point $(x,y)$, then $tpt^{-1}$ is a multivector representing the point $(x+a,y+b)$.
Given this property, we can state that a mapping $\mathcal{L}:\GA\to\GA$ is \SE2 equivariant if for any roto-translation $u\in\GA$ and input $x\in\GA$, we have
\begin{equation}
    \mathcal{L}(uxu^{-1}) = u\mathcal{L}(x)u^{-1}
    \label{eq:equivariance2}
\end{equation}

The encodings in~\reftbl{tbl:encodings} additionally satisfy the properties that
(i) the intersection point of two lines $\ell_1,\ell_2\in\GA$ is given by their wedge product $p:=\ell_1\wedge\ell_2$; and
(ii) the line joining two points $p_1,p_2\in\GA$ is given by a join operator $\ell:=\textrm{Join}(p_1,p_2)$.
For a complete proof of the properties discussed above, please refer to the supplemental material.

\subsection{Equivariant layers}
\label{sec:equivariant_layers}

In this section we define the primitive equivariant layers of our model.
For the full proofs that each of our layers satisfy \refeq{eq:equivariance2}, please refer to the supplemental material. \\\\
\textbf{Linear layers.}
We define $SE(2)$-equivariant linear maps $\phi:\GA\to\GA$ by the form
\begin{equation}
    \phi(x)=\sum_{k=0}^{3}w_k\langle x\rangle_k+\sum_{k=0}^{2}v_ke_0\langle x\rangle_k+\sum_{k=0}^{2}u_ke_{012}\langle x\rangle_k
\end{equation}
where $w_k,v_k,u_k$ are parameters and $\langle\cdot\rangle_k$ denotes blade projection.
Since $\phi$ is linear and the map $x\mapsto uxu^{-1}$ is linear for any fixed operator $u$, to show that $\phi$ is equivariant it suffices to verify that $\phi(uxu^{-1})=u\phi(x)u^{-1}$ for any \textit{basis} multivector $x$ and rotation or translation $u$; please refer to the supplement for details.

We then define affine layers between multivector arrays $\Phi:\GA^{c}\to\GA^{c'}$ via
\begin{equation}
    \Phi(x)_i=\sum_{j=1}^{c}\phi^{(i,j)}(x_j)+b_i\hspace{0.2in}\forall i=1,\ldots,c'
\end{equation}
where $\phi^{(i,j)}$ are equivariant linear layers and $b_i$ are learnable scalar bias terms. \\\\
\textbf{Geometric bilinears.}
To increase expressivity, we include the geometric product which is equivariant as $(uxu^{-1})(uyu^{-1})=u(xy)u^{-1}$ for any $x,y$ and operator $u$, and we include the join operator which is also equivariant~\citet{gatr}.
These are combined to form the equivariant geometric layer $\textrm{Geometric}(w,x,y,z)=\textrm{Concatenate}(wx, \textrm{Join}(y,z))$. \\\\
\textbf{Nonlinearities and normalization.}
We define the equivariant scalar-gated activation $\textrm{GatedRELU}(x)=\textrm{RELU}(\langle x\rangle_0)x$. We define equivariant normalization over channels by $\textrm{LayerNorm}(x) = x / \sqrt{ \mathbb{E}\langle x,x\rangle + \varepsilon}$ where $\langle\cdot,\cdot\rangle$ is the invariant inner product. \\\\
\textbf{Scaled dot-product attention.}
For multivector query and key tokens $q,k\in\GA^C$, the corresponding invariant attention logit may be computed as
\begin{equation}
    \frac{\sum_{c=1}^C\langle q_c,k_c\rangle}{\sqrt{4C}}
\end{equation}
where $\langle\cdot,\cdot\rangle$ is the invariant inner product.
We find it crucial to follow~\citet{gatr} and extend the query and key multivectors with ``distance awareness'' features
\begin{align}
    \phi(q)&=\frac{q_{12}}{q_{12}^2+\varepsilon}\begin{pmatrix}q_{12}^2\\q_{01}^2+q_{20}^2\\q_{01}q_{12}\\q_{20}q_{12}\end{pmatrix}\\
    \psi(k)&=\frac{k_{12}}{k_{12}^2+\varepsilon}\begin{pmatrix}-k_{01}^2-k_{20}^2\\-k_{12}^2\\2k_{01}k_{12}\\2k_{20}k_{12}\end{pmatrix}
\end{align}
which have the property that $\phi(q)\cdot\psi(k)$ is invariant, and moreover if the bivector components of $q$ and $k$ represent points with $q_{12}=k_{12}=1$ then $\phi(q)\cdot\psi(k)=-\frac{1}{(1+\varepsilon)^2}[(k_{01}-q_{01})^2+(k_{20}-q_{20})^2]$ which is proportional to the negative squared distance between the two points.

Combining the multivector features, extended features, and auxiliary scalar features $q^s,k^s\in\mathbb{R}^{C'}$, the invariant attention logits are computed as
\begin{equation}
    \frac{\sum_{c=1}^C\langle q_c,k_c\rangle+\sum_{c=1}^C \phi(q_c)\cdot\psi(k_c)+\sum_{c=1}^{C'}q^s_c k^s_c}{\sqrt{4C+4C+C'}}.
    \label{eq:attention}
\end{equation}
As in a standard transformer, the attention logits are softmaxed and used to take a linear combination of the value tokens.
Notably, \refeq{eq:attention} can be computed as one dot-product by concatenating the three terms for both query and key,
allowing drop-in usage of efficient attention kernels.

\subsection{Invariant adapter}
\label{sec:invariant_adapter}

The output of our model is an invariant action which can be decoded from the auxiliary scalar features.
However, the multivector features contain important geometric information that should be used as well.
To this end, we introduce a novel transform which converts each agent's equivariant geometric features into invariant features by transforming the multivectors into a local coordinate frame.

Specifically, given equivariant multivector features $v\in\GA^{N\times d}$ and auxiliary scalar features $s\in\mathbb{R}^{N\times d'}$, for each agent $n=1,\ldots,N$ we compute
\begin{equation}
    s_{n}\leftarrow\textrm{MLP}(\textrm{flatten\_components}(u_n v_n u_n^{-1}))+s_{n}
    \label{eq:invariant_adapter}
\end{equation}
where $u_n\in\GA$ is the roto-translation transforming global poses into the coordinate frame of agent $n$.
Similar to all other layers,~\refeq{eq:invariant_adapter} is batched across scenes and agent.

For each $n$, $u_n v_n u_n^{-1}$ is invariant because it converts global equivariant features $v_n$ into agent-centric features.
Therefore, this transform preserves the invariance of the auxiliary scalars.

\begin{table*}[t!]
\small
\centering
\begin{tabular}{ l c c c c c c c} 
\toprule
Method                  & \# Params 
                        & Finetuned 
                        & RMM $\uparrow$
                        & Kinematic $\uparrow$
                        & Interactive $\uparrow$
                        & Map-based $\uparrow$
                        & minADE $\downarrow$ \\
\cmidrule(lr){1-3}      \cmidrule(lr){4-8}
BehaviorGPT             & 3M & - & 0.7438 & 0.4254 & 0.7233 & 0.7976 & 1.3804 \\
SMART-7M                & 7M & - & 0.7678 & 0.4894 & 0.7306 & 0.8163 & 1.3532 \\
SMART-7M + CAT-K        & 7M & \checkmark & 0.7709 & 0.4937 & 0.7333 & 0.8185 & 1.2953 \\
\cmidrule(lr){1-3}      \cmidrule(lr){4-8}
Transformer             & 3M & - & 0.7257 & 0.4759 & 0.7058 & 0.7654 & 1.6720 \\
Transformer + RPE       & 3M & - & 0.7251 & 0.4708 & 0.6953 & 0.7808 & 1.7486 \\
Transformer + DRoPE     & 3M & - & 0.7206 & 0.4629 & 0.7017 & 0.7604 & 1.9193 \\
\ourmodel-3M (ours)        & 3M & - & 0.7620 & 0.4859 & 0.7264 & 0.8103 & 1.4192 \\
\ourmodel-30M (ours)        & 30M & - & 0.7636 & 0.4890 & 0.7272 & 0.8120 & 1.3682 \\
\bottomrule
\end{tabular}
\caption{
\textbf{Results on WOSAC 2024 validation set.}
RMM is the Realism Meta Metric used for ranking.
\ourmodel achieves the best RMM among its baselines.
Compared to publicly available state-of-the-art methods, \ourmodel achieves comparable realism
without top-$k$ sampling and closed-loop fine-tuning.
}
\label{table:sota-comparison}
\end{table*}

\section{Experiments}

We evaluate \ourmodel in traffic simulation --- a representative task for modeling traffic agents.
Our experiments demonstrate two key properties:
(1) comparable performance to the state-of-the-art;
and (2) a superior Pareto front in performance \vs computational costs compared to prior work.

\subsection{Comparison to state-of-the-art}

In traffic simulation, equivariant models top the leaderboard but sacrifice computational efficiency
due to their explicit relative positional encodings.
We propose \ourmodel as an alternative that achieves \SE2-equivariance without the high computational costs of existing methods.
Our first experiment seeks to answer the question:
Does \ourmodel match the state-of-the-art?

\paragraph{Dataset \& metrics.}
We evaluate our method on the Waymo Open Motion Dataset (WOMD)~\citep{womd}
using the protocol proposed by the Waymo Open Sim Agents Challenge (WOSAC)~\citep{wosac}.
Given 1s of context in a traffic scenario, the task is to generate 32 rollouts of
8s closed-loop simulations at 10 Hz, each containing the trajectories of every agent in the scenario.
The rollouts are evaluated using distribution matching metrics along three dimensions---
kinematic, interactive, and map-based metrics---which are then summarized into a single
Realism Meta-Metric (RMM) score. We also report the lowest displacement error, averaged over time, across rollouts (minADE).

\paragraph{Baselines.}
We compare against the state-of-the-art on WOSAC:
\textbf{BehaviorGPT}~\citep{behaviorgpt}, \textbf{SMART}~\citep{smart}, and \textbf{CAT-K}~\citep{catk}.
BehaviorGPT and SMART use explicit relative positional encodings (RPEs).
CAT-K further trains SMART with closed-loop fine-tuning.

We also compare against four baselines representative of how state-of-the-art methods achieve \SE2-equivariance.
These baselines share the same base architecture and learning algorithm as \ourmodel,
differing only in how they introduce equivariance.
\textbf{Transformer} is a scene-centric model that uses
roto-translation data augmentation to approximate equivariance and is based on Trajeglish~\citep{trajeglish}.
\textbf{Transformer + DRoPE}~\citep{drope} uses rotary positional encodings (RoPE)
to achieve translation (but not rotation) equivariance.
We faithfully reproduce DRoPE using publicly available information.
\textbf{Transformer + RPE}~\citep{gorela} uses explicit RPEs to achieve equivariance.
Like other methods that use RPEs, we limit each agent's context to a neighbourhood (eight agents and four map tokens)
so that it fits in GPU memory during training.

\paragraph{Implementation details.}
We use a clustered discretized action space.
Starting with a large number of state-to-state transitions, we use a k-disk procedure~\citep{trajeglish} to select a vocabulary of 2048 actions for each agent class (vehicle, cyclist, pedestrian).
We then use this vocabulary to discretize the trajectories in our training dataset, and use a cross entropy loss to train our model to predict the next action for each agent.
We use 16 multivector channels, and six factorized attention blocks.
Our 3M variant uses auxilliary dimension 128 for a total of 2.7 million parameters, and our 30M variant uses auxilliary dimension 512 for 29.4 million parameters.
Similar to prior methods~\citep{smart,catk}, we embed each query token with its prediction at the last timestep as well as its actor class.
For all models and baselines, we train for 250,000 steps 
with learning rate $10^{-3}$ and cosine annealing.
We conduct analysis on the 3M variant, except for the comparison in~\reftbl{table:sota-comparison}, for efficiency purposes.

\paragraph{Results.}
Table \ref{table:sota-comparison} reports results on the WOSAC 2024 validation set.
\ourmodel achieves comparable realism scores to the state-of-the-art despite its relatively small model size.
In fact, \ourmodel achieves a 2\% improvement in RMM over BehaviorGPT---a state-of-the-art model of similar size.
The 30M variant of~\ourmodel achieves comparable RMM with SMART~\citep{smart}, the top method on the WOSAC leaderboard.
For simplicity, we did not explore closed-loop fine-tuning, top-$k$ sampling, and nucleus sampling
in our experiments, but we expect that these techniques will improve performance if tuned properly.

\ourmodel marks a significant improvement over our baselines, despite sharing largely the same base architecture.
From a modeling perspective, \ourmodel has two important advantages over the baselines:
in contrast to Transformer and Transformer + DRoPE, \ourmodel is \SE2-equivariant by construction;
and in contrast to Transformer + RPE, \ourmodel can grow each agent's context to encompass all agents
and map tokens in the scene whereas the baseline is limited to a small neighbourhood due to the
memory footprint of its explicit RPEs.

\begin{figure*}[t]
\centering
\begin{subfigure}[b]{0.3\textwidth}
\centering
\includegraphics[trim={6pt 6pt 6pt 6pt}, clip, width=\textwidth]{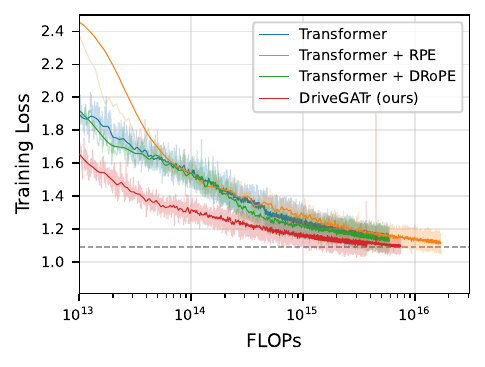}
\end{subfigure}
\begin{subfigure}[b]{0.3\textwidth}
\centering
\includegraphics[trim={6pt 6pt 6pt 6pt}, clip, width=\textwidth]{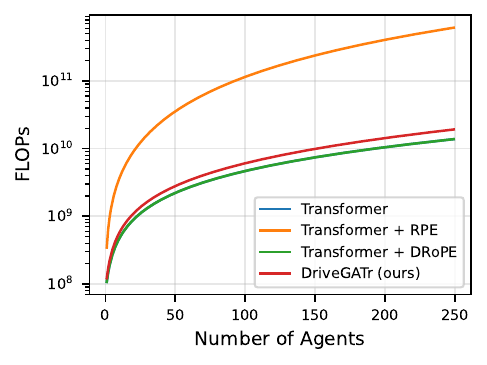}
\end{subfigure}
\begin{subfigure}[b]{0.3\textwidth}
\centering
\includegraphics[trim={6pt 6pt 6pt 6pt}, clip, width=\textwidth]{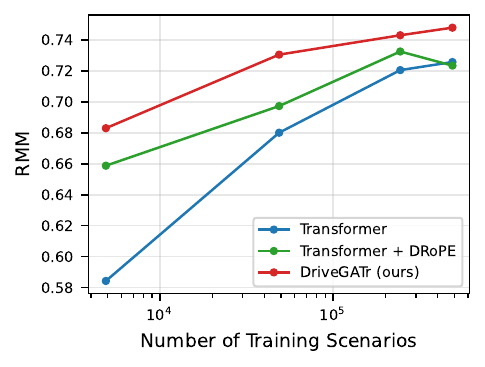}
\end{subfigure}
\caption{
\textbf{Training curves.}
We compare the envelope of minimal training loss per FLOP (\textbf{left}),
compute efficiency as the number of agents scale (\textbf{middle}),
and sample efficiency as the training dataset grows (\textbf{right}).
Compared to the baselines, \ourmodel establishes a superior Pareto front in performance
and computation cost thanks to its compute and sample efficiency.
}
\label{fig:scaling}
\end{figure*}

\subsection{Scaling analysis}
Our next set of experiments evaluate \ourmodel's scalability with respect to training compute and data.

\paragraph{Experiment setup.}
For each of our baselines, we train two models with varying training FLOP profiles on the full WOMD training dataset.
In our experiments, we vary batch size since we found that varying model size, learning rates, \etc was inconsequential.
Then, for each training run, we smooth its training loss curve with Gaussian smoothing.
Finally, from these training loss curves, we extract the envelope of minimal training loss per FLOP.
This training loss envelope gives us an estimate of how each baseline's performance improves as
we scale the available training compute.

\paragraph{Results.}
Figure \ref{fig:scaling} (left) shows that \ourmodel's training loss envelope is significantly lower
than those of the baselines.
For any number of training FLOPs, \ourmodel achieves the lowest training loss,
demonstrating its superior scalability with training compute.
This comes down to \ourmodel 's two key advantages:
(1) higher sample efficiency due to \SE2-equivariance; and
(2) higher compute efficiency due to how it achieves \SE2-equivariance.

To demonstrate \ourmodel 's compute efficiency, we analyze how each model's number of FLOPs scale with
number of agents during inference.
From Figure \ref{fig:scaling} (middle), we see that \ourmodel has significantly lower computational overhead than Transformer + RPE.
The latter requires computing explicit RPE, which introduces an additional overhead that is quadratic in the number of agents.
\ourmodel avoids this overhead by modeling geometric relationships using standard attention between multivectors.

To demonstrate \ourmodel 's sample efficiency, for each model family, we train on a range of dataset sizes: 1\%, 10\%, 50\%, and 100\%.
Next, we plot each model's RMM on the validation split against the number of scenarios seen during training.
Following~\citet{catk}, we evaluate on 2\% of the validation set due to the high cost of evaluation.
In Figure \ref{fig:scaling} (right), we see an intuitive ordering---models with more
inductive bias built in tend to have higher sample efficiency.

\subsection{Additional analysis}
\label{sec:additional_analysis}

\begin{figure*}[t]
\centering
\begin{subfigure}[b]{0.285\textwidth}
\centering
\includegraphics[trim={30cm 12cm 20cm 32cm}, clip,width=\textwidth]{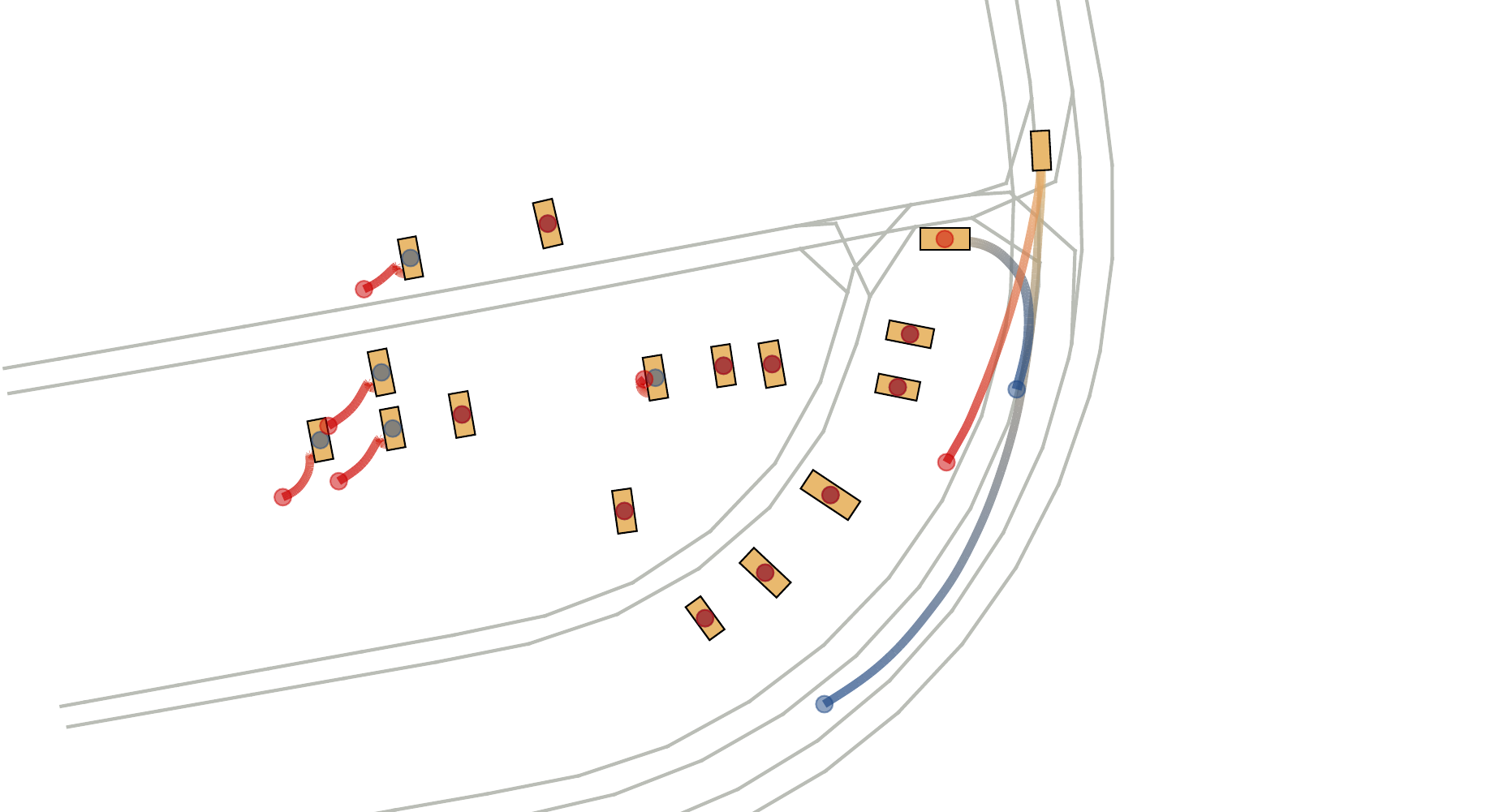}\\
\vspace{0.5cm}
\includegraphics[trim={30cm 20cm 20cm 25cm}, clip,width=\textwidth]{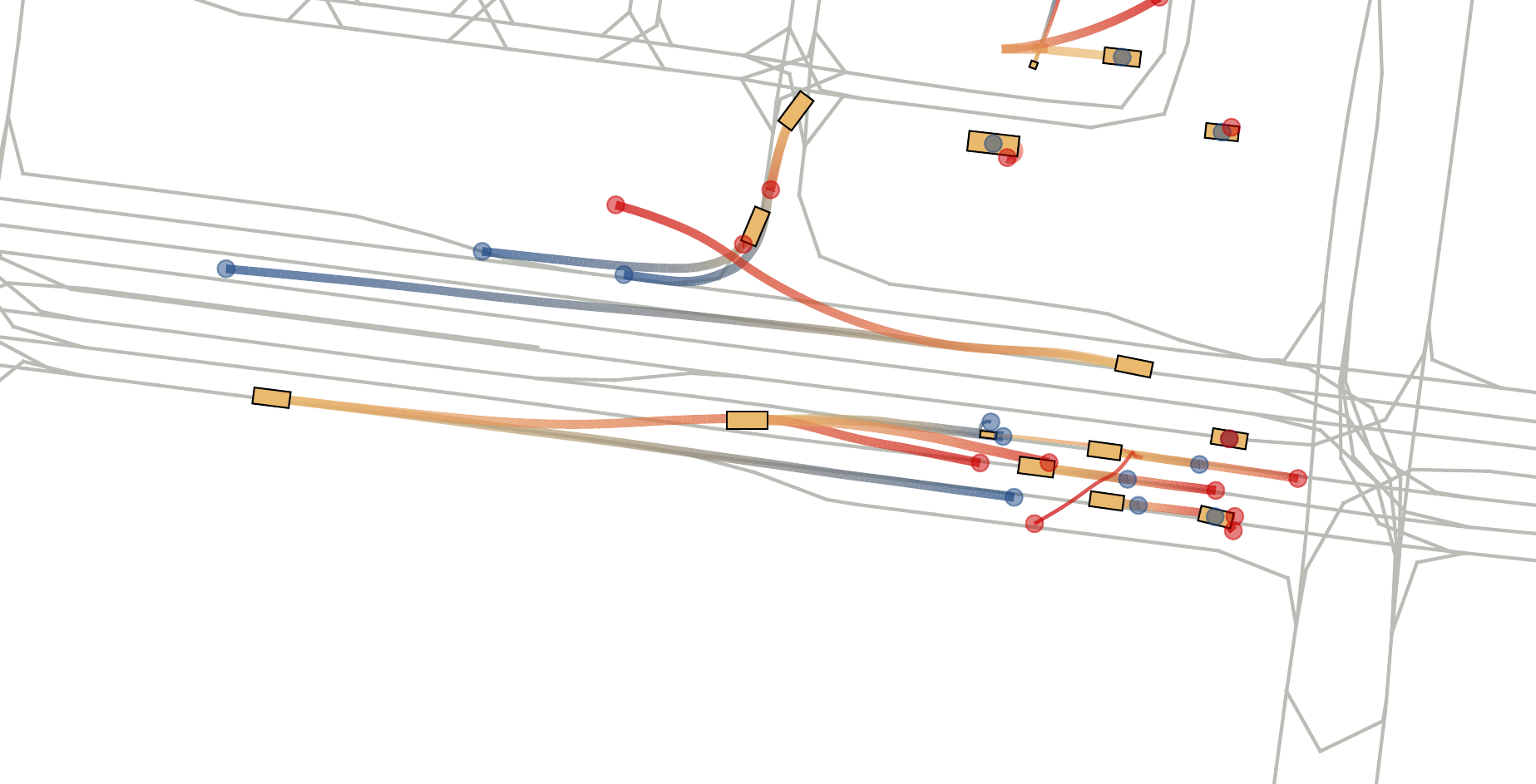}
\end{subfigure}
\begin{subfigure}[b]{0.285\textwidth}
\centering
\includegraphics[trim={30cm 12cm 20cm 32cm}, clip,width=\textwidth]{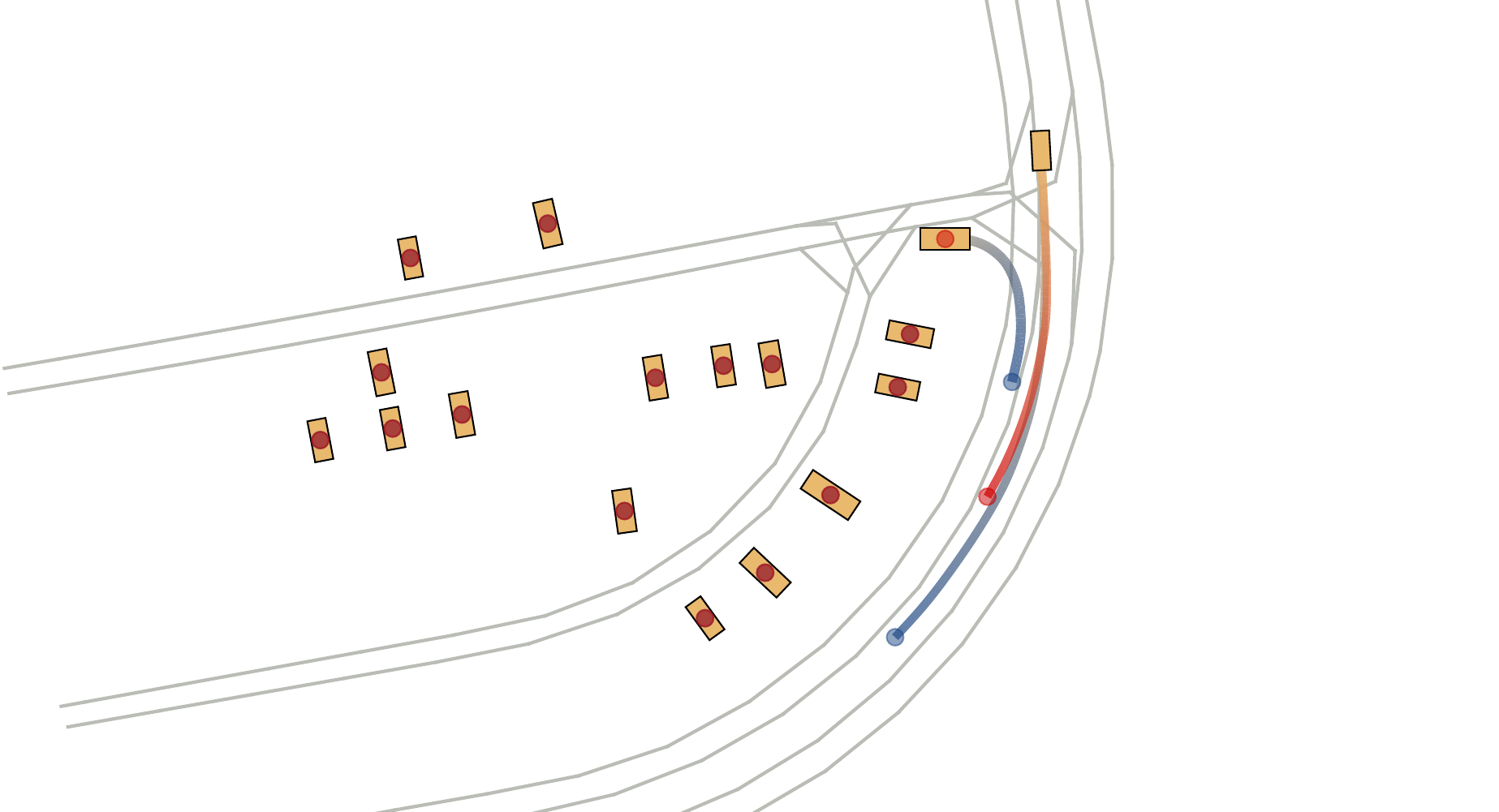}\\
\vspace{0.5cm}
\includegraphics[trim={30cm 20cm 20cm 25cm}, clip,width=\textwidth]{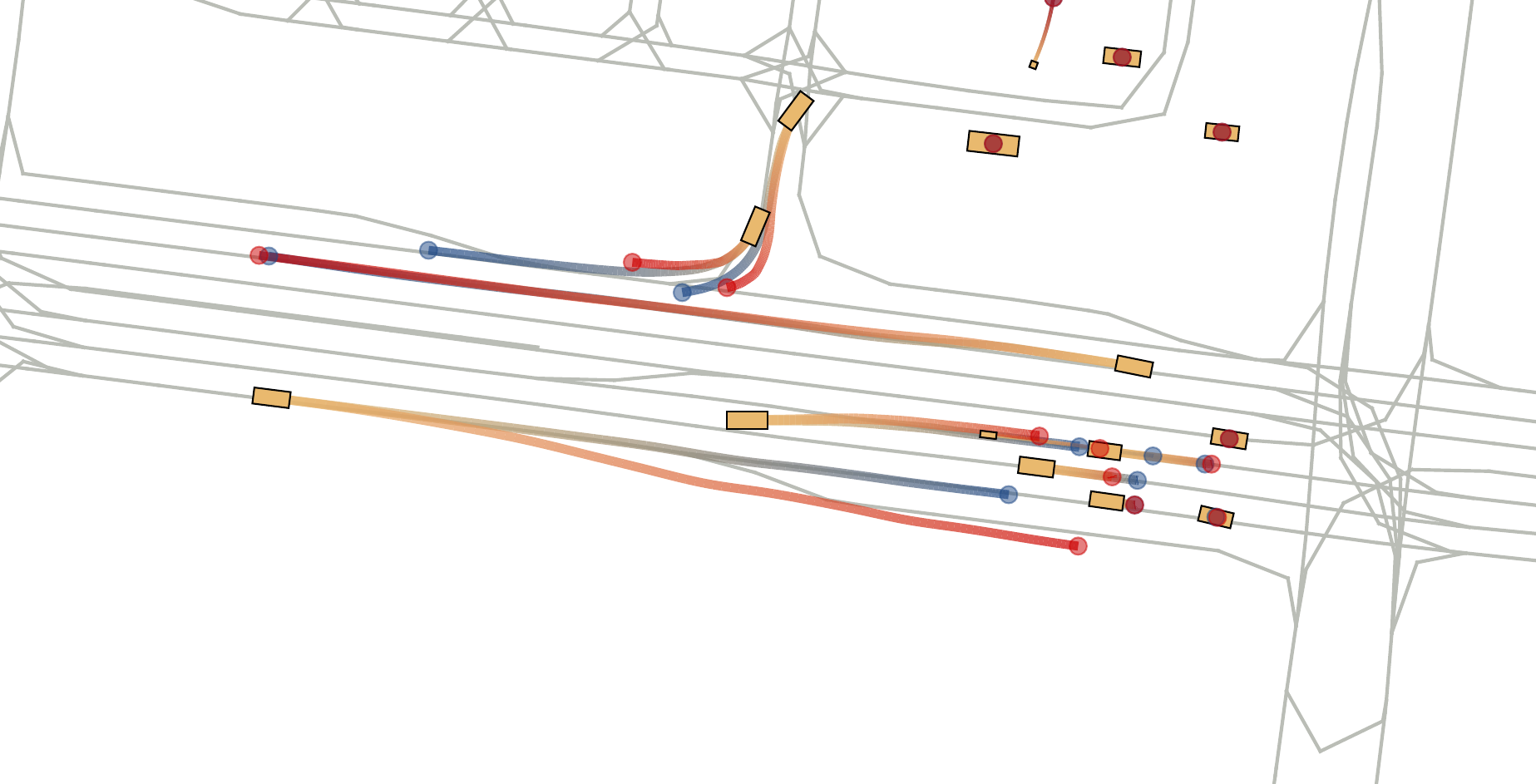}
\end{subfigure}
\begin{subfigure}[b]{0.285\textwidth}
\centering
\includegraphics[trim={30cm 12cm 20cm 32cm}, clip,width=\textwidth]{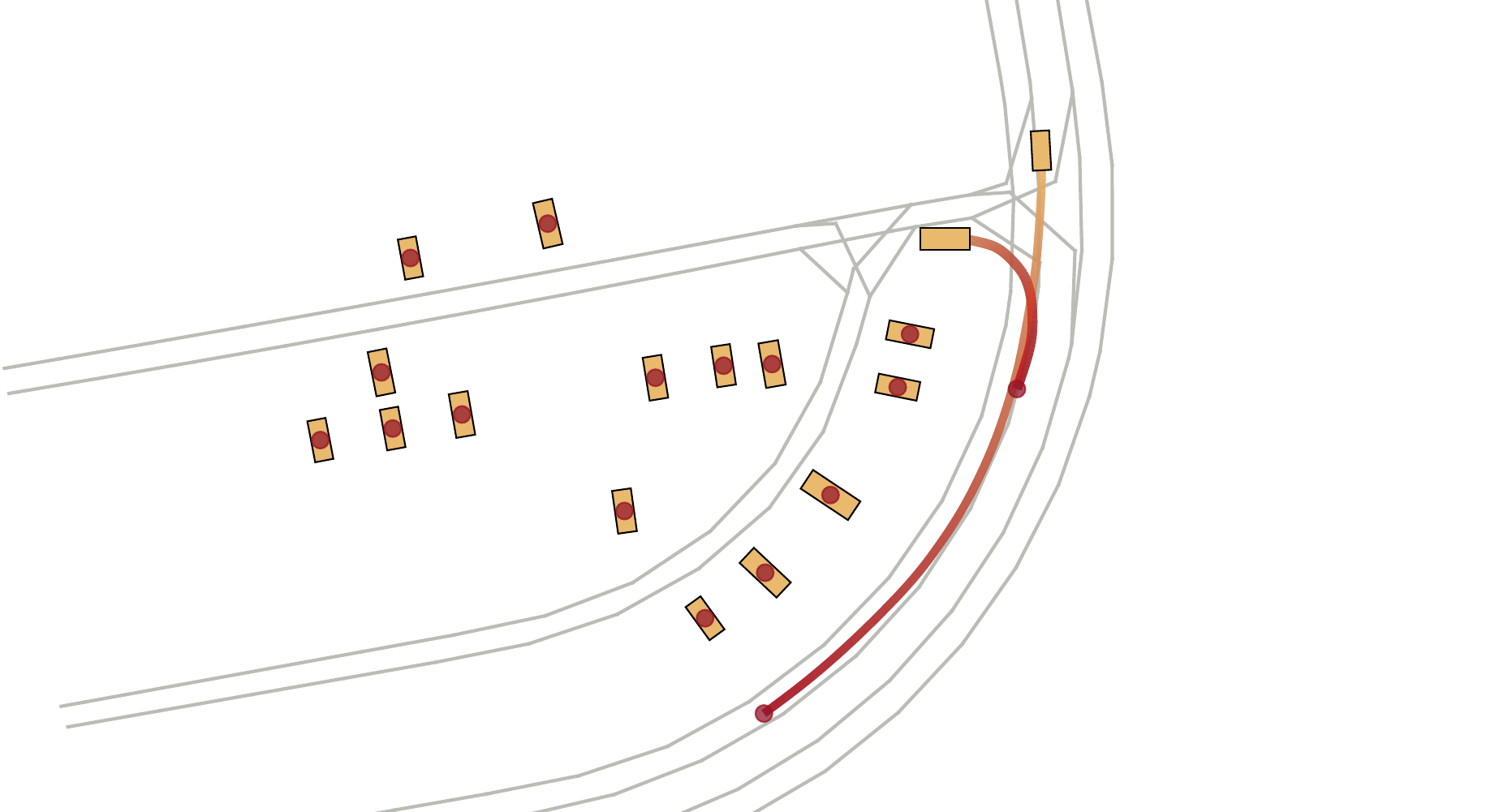}\\
\vspace{0.5cm}
\includegraphics[trim={30cm 20cm 20cm 25cm}, clip,width=\textwidth]{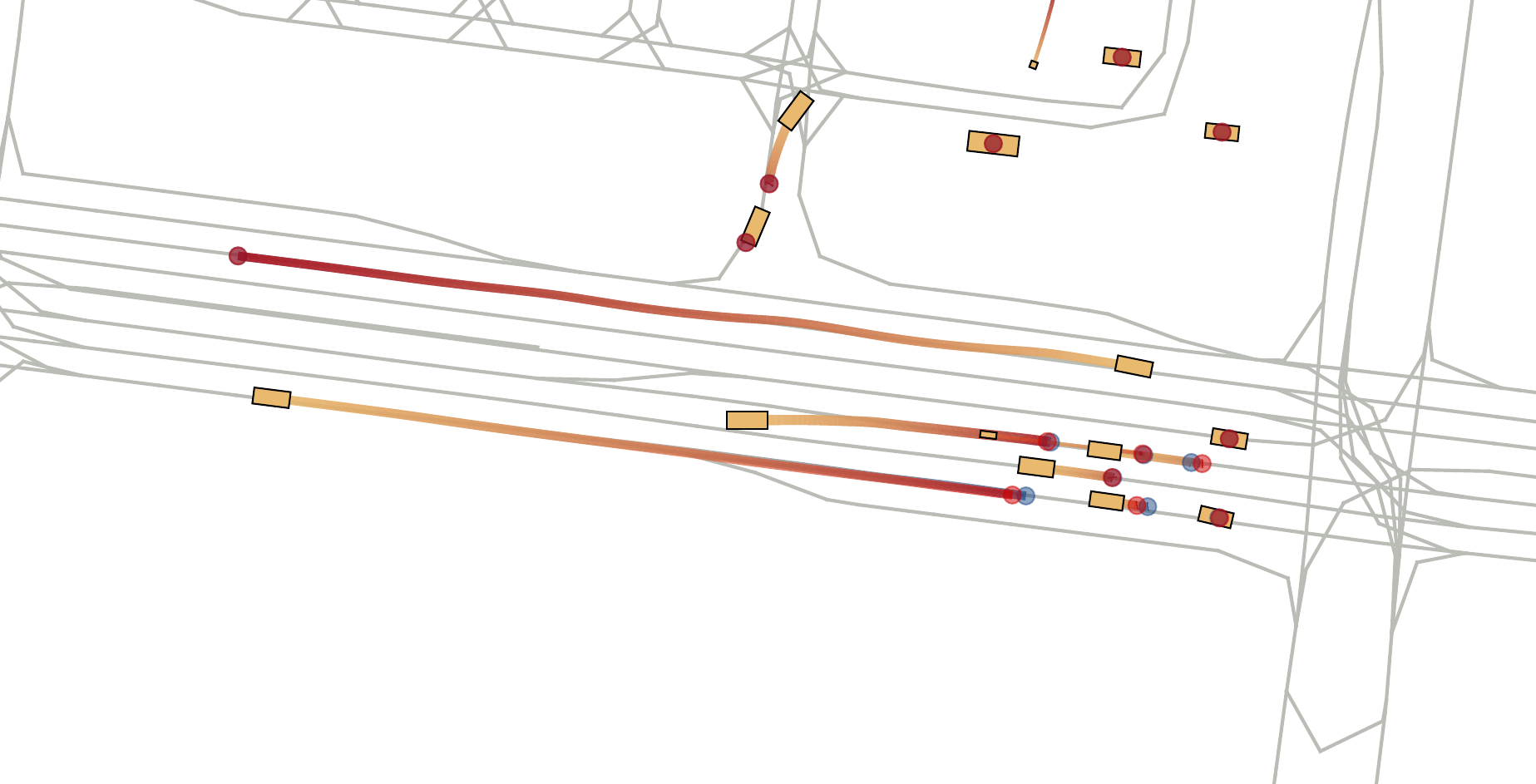}
\end{subfigure}
\caption{
\textbf{Robustness to roto-translations.}
We compare Transformer (\textbf{left}), Transformer + DRoPE (\textbf{middle}), and \ourmodel (\textbf{right})'s
robustness to roto-translations.
In each figure, we overlay rollouts from the original coordinate frame
vs one rotated by 90\textdegree and translated by 100m forward.
Blue trajectories visualize model predictions in the original input, and red visualize predictions in the transformed scene.
\ourmodel produces consistent trajectories despite closed-loop execution, demonstrating its robustness to roto-translations. 
}
\label{fig:robustness}
\end{figure*}

\paragraph{Robustness to roto-translations.}
\SE2-equivariance not only improves sample efficiency but also robustness to nuissance transformations.
In Figure \ref{fig:robustness}, we compare \ourmodel's robustness to roto-translations against Transformer and Transformer + DRoPE.
In each figure, we overlay rollouts from when the model is given inputs in its original coordinate frame \vs
one rotated by 90\textdegree and translated by 100m forward.
Note that for each rollout, instead of sampling, the model at each timestep executes its highest scored action.
As expected, Transformer is not robust to roto-translations---its rollouts change significantly
when given the same inputs from different coordinate frames.
Likewise, Transformer + DRoPE is not robust because it is not rotation equivariant.
In contrast, \ourmodel is robust because it has a mathematical guarantee for \SE2-equivariance.

\begin{table}[t!]
\begin{minipage}{0.45\linewidth}
\scriptsize
\centering
\begin{tabular}{ c c c c c} 
\toprule
IA & DA                 & RMM $\uparrow$
                        & minADE $\downarrow$ \\
\cmidrule(lr){1-2}     \cmidrule(lr){3-4}
\checkmark &            & 0.7408 & 1.6176 \\
           & \checkmark & 0.7483 & 1.6053 \\
\checkmark & \checkmark & 0.7478 & 1.5798 \\
\bottomrule
\end{tabular}
\end{minipage}
\hfill
\begin{minipage}{0.45\linewidth}
\scriptsize
\centering
\begin{tabular}{ l c c} 
\toprule
Map Attn.           & RMM $\uparrow$
                        & minADE $\downarrow$ \\
\cmidrule(lr){1-1}      \cmidrule(lr){2-3}
$ k = 4 $  & 0.7478 & 1.5798 \\
$ k = 8 $  & 0.7528 & 1.5293 \\
All     & 0.7617 & 1.4174 \\
\bottomrule
\end{tabular}
\end{minipage}
\caption{
\textbf{Ablation studies on the WOSAC 2\% validation split.}
On the \textbf{top}, we ablate the importance of the invariant adapter (IA) and distance-awareness (DA).
On the \textbf{bottom}, we ablate the number of map tokens each agent attends to in agent-to-map attention.
}
\label{table:ablation}
\end{table}

\paragraph{Ablation studies.}
In Table \ref{table:ablation}, we ablate two of \ourmodel's key architectural choices on the 2\% validation set, namely (i) whether we use the invariant adapter after each factorized attention block (including the last one)
and (ii) whether we use distance-awareness in \ourmodel's attention layers.
From the left side of Table \ref{table:ablation}, we observe that both architectural choices improve \ourmodel's realism.
For the invariant adapter, we see improvements in minADE only whereas for distance-awareness, we see improvements
across both metrics.
This matches our intuitions that there is important geometric information in the multivector features that should
be mixed into the auxiliary scalars, and that agents should attend to nearby agents and map tokens.

\paragraph{Latency.}
We provide memory footprint and latency comparisons in the supplementary material. We note that \ourmodel is faster and more memory efficient than the fully invariant baseline Transformer + RPE.

\section{Conclusion}

SE(2) equivariance is an inherent symmetry in agents modeling problems.
It is a ubiqutous inductive bias in state-of-the-art agent traffic simulation models,
improving their expressivity, sample efficiency, and robustness to nuissance transformations.
However, existing approaches achieve this with high computational costs.
We propose \ourmodel, a novel architecture for modeling traffic agents that guarantees
SE(2) equivariance yet avoids the high computational costs of existing approaches.
By leveraging efficient 2D geometric algebra encodings and the equivariant layer primitives,
\ourmodel is expressive, efficient and provably equivariant.

\paragraph{Limitations and broader impact.}
Although $\SE2$ is a useful equivariance property for agent modeling, self-driving is ultimately a 3D problem.
\ourmodel models equivariance in 2D but, like prior work~\citep{gigaflow}, it can be easily extended
to 2.5D by adding the height dimension to the auxiliary scalar features.
Another possibility is combining our method of acheiving $\SE2$ equivariance in the $xy$-plane, with translational invariance in the $z$-axis.
While we only evaluate \ourmodel in traffic simulation, we believe it is extensible to similar tasks like motion forecasting and ML-based planning.
We also highlight that \ourmodel is capable of making not just invariant but equivariant predictions in the global coordinate frame,
making it a suitable architectural choice for more challenging problems such as traffic scenario generation.
We leave this direction for future work.

{
    \small
    \bibliographystyle{ieeenat_fullname}
    \bibliography{main}
}

\clearpage
\onecolumn

\begin{center}
    {\Large\bfseries Supplemental Materials for Efficient Equivariant Transformer for Self-Driving Agent Modeling}
    \vspace{1em}
\end{center}

\section{Theoretical analysis}
We provide additional theoretical analysis in the following sections.
\refsec{sec:additional_ga_preliminaries} provides additional preliminaries for 2D projective geometric algebra $\GA$.
\refsec{sec:encoding_proofs} provides proofs of the satisfied properties of our derived $\GA$ encodings.
\refsec{sec:equivariance_proofs} provides proofs for $\SE2$ equivariance for each \ourmodel layer.

\subsection{Additional geometric algebra preliminaries}
In the Euclidean space $\mathbb{R}^n$, one may define a bilinear product of two vectors $x$ and $y$, called the \textit{geometric product} $xy$, characterized by the fundamental equation $v^2=\langle v,v\rangle$ where $\langle\cdot,\cdot\rangle$ is the standard inner product.
From the fundamental equation one can deduce that the geometric product is anticommutative on orthogonal basis vectors: for $i\neq j$, $(e_i+e_j)^2=\langle e_i+e_j,e_i+e_j\rangle\implies e_i^2+e_ie_j+e_je_i+e_j^2=e_i^2+2\langle e_i,e_j\rangle+e_j^2\implies e_ie_j=-e_je_i$.
Multiplying vectors together give us \textit{multivectors}; for $n=2$, the Euclidean geometric algebra consists of multivectors of the form $x=x'+x_1e_1+x_2e_2+x_{12}e_{12}$, where $x',x_1,x_2,x_{12}$ are real coefficients and $e_{12}:=e_1e_2$.

To get a \textit{projective} geometric algebra, we adjoin to $\mathbb{R}^n$ an orthogonal basis vector $e_0$ with norm zero, that is, $e_0^2=0$.
For $n=2$, the resulting algebra, denoted by $\GA$, is 8-dimensional, spanned by the scalar 1, three vectors $e_0,e_1,e_2$, three \textit{bivectors} $e_{01},e_{20},e_{12}$, and one \textit{pseudoscalar} $e_{012}$.
Using the properties that $e_0^2=0, e_1^2=e_2^2=1, e_{ij}=-e_{ji}$ for $i\neq j$, one can compute a table of geometric products $xy$:
\begin{center}
\begin{tabular}{c||c|c|c|c|c|c|c|c}
    $x$\textbackslash $y$ & $1$ & $e_{0}$ & $e_{1}$ & $e_{2}$ & $e_{01}$ & $e_{20}$ & $e_{12}$ & $e_{012}$ \\\hline\hline
    $1$ & $1$ & $e_{0}$ & $e_{1}$ & $e_{2}$ & $e_{01}$ & $e_{20}$ & $e_{12}$ & $e_{012}$ \\\hline
    $e_{0}$ & $e_{0}$ & $0$ & $e_{01}$ & $-e_{20}$ & $0$ & $0$ & $e_{012}$ & $0$ \\\hline
    $e_{1}$ & $e_{1}$ & $-e_{01}$ & $1$ & $e_{12}$ & $-e_{0}$ & $e_{012}$ & $e_{2}$ & $e_{20}$ \\\hline
    $e_{2}$ & $e_{2}$ & $e_{20}$ & $-e_{12}$ & $1$ & $e_{012}$ & $e_{0}$ & $-e_{1}$ & $e_{01}$ \\\hline
    $e_{01}$ & $e_{01}$ & $0$ & $e_{0}$ & $e_{012}$ & $0$ & $0$ & $-e_{20}$ & $0$ \\\hline
    $e_{20}$ & $e_{20}$ & $0$ & $e_{012}$ & $-e_{0}$ & $0$ & $0$ & $e_{01}$ & $0$ \\\hline
    $e_{12}$ & $e_{12}$ & $e_{012}$ & $-e_{2}$ & $e_{1}$ & $e_{20}$ & $-e_{01}$ & $-1$ & $-e_{0}$ \\\hline
    $e_{012}$ & $e_{012}$ & $0$ & $e_{20}$ & $e_{01}$ & $0$ & $0$ & $-e_{0}$ & $0$
\end{tabular}
\end{center}

A \textit{sandwich product} is an expression of the form $uxu^{-1}$, where $u,x$ are multivectors and $u$ has a multiplicative inverse with respect to the geometric product.

We also define the \textit{wedge product} $x\wedge y$, which is another bilinear product, but is defined by the fundamental equation $v\wedge v=0$.
One can derive that $e_i\wedge e_j=-e_j\wedge e_i$ for orthogonal basis vectors and compute a table of wedge products $x\wedge y$:
\begin{center}
\begin{tabular}{c||c|c|c|c|c|c|c|c}
    $x$\textbackslash $y$ & $1$ & $e_{0}$ & $e_{1}$ & $e_{2}$ & $e_{01}$ & $e_{20}$ & $e_{12}$ & $e_{012}$ \\\hline\hline
    $1$ & $1$ & $e_{0}$ & $e_{1}$ & $e_{2}$ & $e_{01}$ & $e_{20}$ & $e_{12}$ & $e_{012}$ \\\hline
    $e_{0}$ & $e_{0}$ & $0$ & $e_{01}$ & $-e_{20}$ & $0$ & $0$ & $e_{012}$ & $0$ \\\hline
    $e_{1}$ & $e_{1}$ & $-e_{01}$ & $0$ & $e_{12}$ & $0$ & $e_{012}$ & $0$ & $0$ \\\hline
    $e_{2}$ & $e_{2}$ & $e_{20}$ & $-e_{12}$ & $0$ & $e_{012}$ & $0$ & $0$ & $0$ \\\hline
    $e_{01}$ & $e_{01}$ & $0$ & $0$ & $e_{012}$ & $0$ & $0$ & $0$ & $0$ \\\hline
    $e_{20}$ & $e_{20}$ & $0$ & $e_{012}$ & $0$ & $0$ & $0$ & $0$ & $0$ \\\hline
    $e_{12}$ & $e_{12}$ & $e_{012}$ & $0$ & $0$ & $0$ & $0$ & $0$ & $0$ \\\hline
    $e_{012}$ & $e_{012}$ & $0$ & $0$ & $0$ & $0$ & $0$ & $0$ & $0$
\end{tabular}
\end{center}
The \textit{multivector dual} $x\mapsto x^*$ is a linear operator, defined for a basis multivector $x$ as the basis multivector such that $x\wedge x^*=e_{012}$.
For a general $\GA$ multivector
\begin{equation*}
    x=x'+x_0e_0+x_1e_1+x_2e_2+x_{01}e_{01}+x_{20}e_{20}+x_{12}e_{12}+x_{012}e_{012}
\end{equation*}
we have
\begin{equation*}
    x^*:=x_{012}+x_{12}e_0+x_{20}e_1+x_{01}e_2+x_2e_{01}+x_1e_{20}+x_0e_{12}+x'e_{012}
\end{equation*}
which has the same coefficients as $x$, in reverse order.

The \textit{join} operator of two multivectors is defined as $\textrm{Join}(x,y)=(x^*\wedge y^*)^*$.

Finally, we define \textit{$k$-grade projection} $\langle\cdot\rangle_k$ for $k=0,1,2,3$, which selects specific components of $x$:
\begin{align*}
    \langle x\rangle_0&:=x'\\
    \langle x\rangle_1&:=x_0e_0+x_1e_1+x_2e_2\\
    \langle x\rangle_2&:=x_{01}e_{01}+x_{20}e_{20}+x_{12}e_{12}\\
    \langle x\rangle_3&:=x_{012}e_{012}
\end{align*}
and we define the \textit{invariant inner product} $\langle\cdot,\cdot\rangle$ between two multivectors:
\begin{equation*}
    \langle x,y\rangle:=x'y'+x_1y_1+x_2y_2+x_{12}y_{12}\in\mathbb{R}
\end{equation*}
which ignores components containing $e_0$.
\label{sec:additional_ga_preliminaries}

\subsection{Encodings}
We encode 2D objects and transformations into $\GA$ using Table 1.
These encodings satisfy many nice properties described in subsequent propositions.

\begin{proposition}
    The result of applying an operator $u$ to an object $x$ is given by a sandwich product $u[x]:=uxu^{-1}$.
\end{proposition}

\begin{remark}
    Note that for \SE2-equivariance, we have only considered when $u$ is a non-mirroring transformation.
    If $u$ is a mirroring transformation (e.g. a reflection), one must instead use $u[x]:=u\hat{x}u^{-1}$ where $\hat{x}$ denotes grade involution, which flips the sign of the odd-grade components of $x$.
\end{remark}

\begin{proof}
Since geometric product is bilinear, then $u[\cdot]$ is linear for any fixed $u$. As such, it is helpful to first compute sandwich products $u[x]$ for \textit{basis} multivectors $x$.
When $u=t$ is a translation by $(a,b)$,
\begin{align*}
    t[1]&=1\\
    t[e_0]&=(1-\tfrac{a}{2}e_{01}+\tfrac{b}{2}e_{20})e_0(1+\tfrac{a}{2}e_{01}-\tfrac{b}{2}e_{20})\\
    &=e_0\\
    t[e_1]&=(1-\tfrac{a}{2}e_{01}+\tfrac{b}{2}e_{20})e_1(1+\tfrac{a}{2}e_{01}-\tfrac{b}{2}e_{20})\\
    &=e_1+\tfrac{a}{2}e_1e_{01}-\tfrac{b}{2}e_1e_{20}-\tfrac{a}{2}e_{01}e_1+\tfrac{b}{2}e_{20}e_1\\
    &=e_1-ae_0\\
    t[e_2]&=(1-\tfrac{a}{2}e_{01}+\tfrac{b}{2}e_{20})e_2(1+\tfrac{a}{2}e_{01}-\tfrac{b}{2}e_{20})\\
    &=e_2+\tfrac{a}{2}e_2e_{01}-\tfrac{b}{2}e_2e_{20}-\tfrac{a}{2}e_{01}e_2+\tfrac{b}{2}e_{20}e_2\\
    &=e_2-be_0\\
    t[e_{01}]&=(1-\tfrac{a}{2}e_{01}+\tfrac{b}{2}e_{20})e_{01}(1+\tfrac{a}{2}e_{01}-\tfrac{b}{2}e_{20})\\
    &=e_{01}\\
    t[e_{20}]&=(1-\tfrac{a}{2}e_{01}+\tfrac{b}{2}e_{20})e_{20}(1+\tfrac{a}{2}e_{01}-\tfrac{b}{2}e_{20})\\
    &=e_{20}\\
    t[e_{12}]&=(1-\tfrac{a}{2}e_{01}+\tfrac{b}{2}e_{20})e_{12}(1+\tfrac{a}{2}e_{01}-\tfrac{b}{2}e_{20})\\
    &=e_{12}+\tfrac{a}{2}e_{12}e_{01}-\tfrac{b}{2}e_{12}e_{20}-\tfrac{a}{2}e_{01}e_{12}+\tfrac{b}{2}e_{20}e_{12}\\
    &=e_{12}+be_{01}+ae_{20}\\
    t[e_{012}]&=(1-\tfrac{a}{2}e_{01}+\tfrac{b}{2}e_{20})e_{012}(1+\tfrac{a}{2}e_{01}-\tfrac{b}{2}e_{20})\\
    &=e_{012}
\end{align*}
When $u=r$ is a rotation by $\theta$,
\begin{align*}
    r[1]&=1\\
    r[e_0]&=(\cos\tfrac{\theta}{2}-\sin\tfrac{\theta}{2}e_{12})e_0(\cos\tfrac{\theta}{2}+\sin\tfrac{\theta}{2}e_{12})\\
    &=\cos^2\tfrac{\theta}{2}e_0+\cos\tfrac{\theta}{2}\sin\tfrac{\theta}{2}e_0e_{12}-\sin\tfrac{\theta}{2}\cos\tfrac{\theta}{2}e_{12}e_0-\sin^2\tfrac{\theta}{2}e_{12}e_0e_{12}\\
    &=e_0\\
    r[e_1]&=(\cos\tfrac{\theta}{2}-\sin\tfrac{\theta}{2}e_{12})e_1(\cos\tfrac{\theta}{2}+\sin\tfrac{\theta}{2}e_{12})\\
    &=\cos^2\tfrac{\theta}{2}e_1+\cos\tfrac{\theta}{2}\sin\tfrac{\theta}{2}e_1e_{12}-\sin\tfrac{\theta}{2}\cos\tfrac{\theta}{2}e_{12}e_1-\sin^2\tfrac{\theta}{2}e_{12}e_1e_{12}\\
    &=\cos\theta e_1+\sin\theta e_2\\
    r[e_2]&=(\cos\tfrac{\theta}{2}-\sin\tfrac{\theta}{2}e_{12})e_2(\cos\tfrac{\theta}{2}+\sin\tfrac{\theta}{2}e_{12})\\
    &=\cos^2\tfrac{\theta}{2}e_2+\cos\tfrac{\theta}{2}\sin\tfrac{\theta}{2}e_2e_{12}-\sin\tfrac{\theta}{2}\cos\tfrac{\theta}{2}e_{12}e_2-\sin^2\tfrac{\theta}{2}e_{12}e_2e_{12}\\
    &=\cos\theta e_2-\sin\theta e_1\\
    r[e_{01}]&=(\cos\tfrac{\theta}{2}-\sin\tfrac{\theta}{2}e_{12})e_{01}(\cos\tfrac{\theta}{2}+\sin\tfrac{\theta}{2}e_{12})\\
    &=\cos^2\tfrac{\theta}{2}e_{01}+\cos\tfrac{\theta}{2}\sin\tfrac{\theta}{2}e_{01}e_{12}-\sin\tfrac{\theta}{2}\cos\tfrac{\theta}{2}e_{12}e_{01}-\sin^2\tfrac{\theta}{2}e_{12}e_{01}e_{12}\\
    &=\cos\theta e_{01}-\sin\theta e_{20}\\
    r[e_{20}]&=(\cos\tfrac{\theta}{2}-\sin\tfrac{\theta}{2}e_{12})e_{20}(\cos\tfrac{\theta}{2}+\sin\tfrac{\theta}{2}e_{12})\\
    &=\cos^2\tfrac{\theta}{2}e_{20}+\cos\tfrac{\theta}{2}\sin\tfrac{\theta}{2}e_{20}e_{12}-\sin\tfrac{\theta}{2}\cos\tfrac{\theta}{2}e_{12}e_{20}-\sin^2\tfrac{\theta}{2}e_{12}e_{20}e_{12}\\
    &=\sin\theta e_{01}+\cos\theta e_{20}\\
    r[e_{12}]&=(\cos\tfrac{\theta}{2}-\sin\tfrac{\theta}{2}e_{12})e_{12}(\cos\tfrac{\theta}{2}+\sin\tfrac{\theta}{2}e_{12})\\
    &=\cos^2\tfrac{\theta}{2}e_{12}+\cos\tfrac{\theta}{2}\sin\tfrac{\theta}{2}e_{12}e_{12}-\sin\tfrac{\theta}{2}\cos\tfrac{\theta}{2}e_{12}e_{12}-\sin^2\tfrac{\theta}{2}e_{12}e_{12}e_{12}\\
    &=e_{12}\\
    r[e_{012}]&=(\cos\tfrac{\theta}{2}-\sin\tfrac{\theta}{2}e_{12})e_{012}(\cos\tfrac{\theta}{2}+\sin\tfrac{\theta}{2}e_{12})\\
    &=\cos^2\tfrac{\theta}{2}e_{012}+\cos\tfrac{\theta}{2}\sin\tfrac{\theta}{2}e_{012}e_{12}-\sin\tfrac{\theta}{2}\cos\tfrac{\theta}{2}e_{12}e_{012}-\sin^2\tfrac{\theta}{2}e_{12}e_{012}e_{12}\\
    &=e_{012}
\end{align*}
Using the above, we can now verify that
\begin{enumerate}
    \item A translation by $(a,b)\in\mathbb{R}^2$ of the point $p=(x,y)$ yields the point $(x',y')=(x+a,y+b)$:
    \begin{align*}
        t[p]&=t[xe_{20}+ye_{01}+e_{12}]\\
        &=x\cdot t[e_{20}]+y\cdot t[e_{01}]+t[e_{12}]\\
        &=xe_{20}+ye_{01}+(e_{12}+be_{01}+ae_{20})\\
        &=x'e_{20}+y'e_{01}+e_{12}
    \end{align*}
    \item A translation by $(a,b)\in\mathbb{R}^2$ of the line $\ell:Ax+By+C=0$ yields the line $Ax+By+C'=0$, where $C'=C-Aa-Bb$:
    \begin{align*}
        t[\ell]&=t[Ae_1+Be_2+Ce_0]\\
        &=A\cdot t[e_1]+B\cdot t[e_2]+C\cdot t[e_0]\\
        &=A(e_1-ae_0)+B(e_2-be_0)+Ce_0\\
        &=Ae_1+Be_2+C'e_0
    \end{align*}
    \item A rotation by angle $\theta\in\mathbb{R}$ of the point $p=(x,y)$ yields the point $(x',y')=(x\cos\theta-y\sin\theta,x\sin\theta+y\cos\theta)$:
    \begin{align*}
        r[p]&=r[xe_{20}+ye_{01}+e_{12}]\\
        &=x\cdot r[e_{20}]+y\cdot r[e_{01}]+r[e_{12}]\\
        &=x(\sin\theta e_{01}+\cos\theta e_{20})+y(\cos\theta e_{01}-\sin\theta e_{20})+e_{12}\\
        &=x'e_{20}+y'e_{01}+e_{12}
    \end{align*}
    \item A rotation by angle $\theta\in\mathbb{R}$ of the line $\ell:Ax+By+C=0$ yields the line $A'x+B'y+C=0$, where $(A',B')=(A\cos\theta-B\sin\theta,A\sin\theta+B\cos\theta)$:
    \begin{align*}
        r[\ell]&=r[Ae_1+Be_2+Ce_0]\\
        &=A\cdot r[e_1]+B\cdot r[e_2]+C\cdot r[e_0]\\
        &=A(\cos\theta e_1+\sin\theta e_2)+B(\cos\theta e_2-\sin\theta e_1)+Ce_0\\
        &=A'e_1+B'e_2+Ce_0
    \end{align*}
\end{enumerate}
This concludes the proof.
\end{proof}

\begin{proposition}
The intersection of two lines is given by a wedge product.
\end{proposition}

\begin{proof}
For lines $a_1x+b_1y+c_1=0$ and $a_2x+b_2y+c_2=0$, their intersection point $x_0=\cfrac{b_1c_2-b_2c_1}{a_1b_2-a_2b_1}$, $y_0=\cfrac{a_2c_1-a_1c_2}{a_1b_2-a_2b_1}$ can be computed via
\begin{align*}
    \ell_1\wedge\ell_2&=(a_1e_1+b_1e_2+c_1e_0)\wedge(a_2e_1+b_2e_2+c_2e_0)\\
    &=(a_2c_1-a_1c_2)e_{01}+(b_1c_2-b_2c_1)e_{20}+(a_1b_2-a_2b_1)e_{12}\\
    &=(a_1b_2-a_2b_1)(x_0e_{20}+y_0e_{01}+e_{12})
\end{align*}
\end{proof}

\begin{proposition}
    The line joining two points is given by a join operator.
\end{proposition}

\begin{proof}
If we join the points $(a,b)$ and $(c,d)$, we obtain
\begin{align*}
    &((ae_{20}+be_{01}+e_{12})^*\wedge(ce_{20}+de_{01}+e_{12})^*)^*\\
    &=((ae_1+be_2+e_0)\wedge(ce_1+de_2+e_0))^*\\
    &=((c-a)e_{01}+(b-d)e_{20}+(ad-bc)e_{12})^*\\
    &=(c-a)e_2+(b-d)e_1+(ad-bc)e_0
\end{align*}
which is the encoding of the line $(b-d)x+(c-a)y+(ad-bc)=0$ passing through both points.
\end{proof}

\begin{proposition}
The distance from a point to a line is given by a join operator.
\end{proposition}

\begin{proof}
If we join the point $(x_0,y_0)$ and the line $ax+by+c=0$ where $\sqrt{a^2+b^2}=1$, we obtain
\begin{align*}
    &((x_0e_{20}+y_0e_{01}+e_{12})^*\wedge(ae_1+be_2+ce_0)^*)^*\\
    &=((x_0e_1+y_0e_2+e_0)\wedge(ae_{20}+be_{01}+ce_{12}))^*\\
    &=((ax_0+by_0+c)e_{012})^*\\
    &=ax_0+by_0+c
\end{align*}
which is the signed distance between them.
\end{proof}
\label{sec:encoding_proofs}

\subsection{Equivariant layers}

To show that a function $f:\GA\to\GA$ is \SE2-equivariant, one can verify that for any rotation or translation operator $u$ and input multivector $x$, we have $f(u[x])=u[f(x)]$.

We define linear maps $\phi:\GA\to\GA$ in the form
\begin{equation}
    \phi(x)=\sum_{k=0}^{3}w_k\langle x\rangle_k+\sum_{k=0}^{2}v_ke_0\langle x\rangle_k+\sum_{k=0}^{2}u_ke_{012}\langle x\rangle_k
\end{equation}
Note that this differs from the linear layer defined in ~\citep{gatr} as we have introduced additional parameters $u_k$.
With these extra terms, $\phi$ is not equivariant to reflections.

This is how $\phi$ acts on the basis multivectors:
\begin{align*}
    \phi(1)&=w_0+v_0e_0+u_0e_{012}\\
    \phi(e_0)&=w_1e_0\\
    \phi(e_1)&=w_1e_1+v_1e_{01}+u_1e_{20}\\
    \phi(e_2)&=w_1e_2-v_1e_{20}+u_1e_{01}\\
    \phi(e_{01})&=w_2e_{01}\\
    \phi(e_{20})&=w_2e_{20}\\
    \phi(e_{12})&=w_2e_{12}+v_2e_{012}-u_2e_0\\
    \phi(e_{012})&=w_3e_{012}
\end{align*}

\begin{proposition}
The linear map defined in \textnormal{equation (20)} is \SE2-equivariant.
\end{proposition}

\begin{proof}
Let $t=1-\tfrac{a}{2}e_{01}+\tfrac{b}{2}e_{20}$ be a translation operator and $r=\cos\frac{\theta}{2}-\sin\frac{\theta}{2}e_{12}$ be a rotation operator.
Since $\phi$ and $u[\cdot]$ are linear, it suffices to show that $t[\phi(x)]=\phi(t[x])$ and $r[\phi(x)]=\phi(r[x])$ for the basis multivectors $x\in\{1,e_0,e_1,e_2,e_{01},e_{20},e_{12},e_{012}\}$.

Referring to the proof of Proposition 1 for values of $t[x]$ and $r[x]$, we can verify these equations hold for translations:
\begin{align*}
    \phi(t[1])&=\phi(1)=w_0+v_0e_0+u_0e_{012}\\
    t[\phi(1)]&=t[w_0+v_0e_0+u_0e_{012}]=w_0+v_0e_0+u_0e_{012}\\[5pt]
    \phi(t[e_0])&=\phi(e_0)=w_1e_0\\
    t[\phi(e_0)]&=t[w_1e_0]=w_1e_0\\[5pt]
    \phi(t[e_1])&=\phi(e_1-ae_0)=(w_1e_1+v_1e_{01}+u_1e_{20})-aw_1e_0\\
    t[\phi(e_1)]&=t[w_1e_1+v_1e_{01}+u_1e_{20}]=w_1(e_1-ae_0)+v_1e_{01}+u_1e_{20}\\[5pt]
    \phi(t[e_2])&=\phi(e_2-be_0)=(w_1e_2-v_1e_{20}+u_1e_{01})-bw_1e_0\\
    t[\phi(e_2)]&=t[w_1e_2-v_1e_{20}+u_1e_{01}]=w_1(e_2-be_0)-v_1e_{20}+u_1e_{01}\\[5pt]
    \phi(t[e_{01}])&=\phi(e_{01})=w_2e_{01}\\
    t[\phi(e_{01})]&=t[w_2e_{01}]=w_2e_{01}\\[5pt]
    \phi(t[e_{20}])&=\phi(e_{20})=w_2e_{20}\\
    t[\phi(e_{20})]&=t[w_2e_{20}]=w_2e_{20}\\[5pt]
    \phi(t[e_{12}])&=\phi(e_{12}+be_{01}+ae_{20})=(w_2e_{12}+v_2e_{012}-u_2e_0)+bw_2e_{01}+aw_2e_{20}\\
    t[\phi(e_{12})]&=t[w_2e_{12}+v_2e_{012}-u_2e_0]=w_2(e_{12}+be_{01}+ae_{20})+v_2e_{012}-u_2e_0\\[5pt]
    \phi(t[e_{012}])&=\phi(e_{012})=w_3e_{012}\\
    t[\phi(e_{012})]&=t[w_3e_{012}]=w_3e_{012}
\end{align*}
and rotations:
\begin{align*}
    \phi(r[1])&=\phi(1)=w_0+v_0e_0+u_0e_{012}\\
    r[\phi(1)]&=r[w_0+v_0e_0+u_0e_{012}]=w_0+v_0e_0+u_0e_{012}\\[5pt]
    \phi(r[e_0])&=\phi(e_0)=w_1e_0\\
    r[\phi(e_0)]&=r[w_1e_0]=w_1e_0\\[5pt]
    \phi(r[e_1])&=\phi(\cos\theta e_1+\sin\theta e_2)=\cos\theta(w_1e_1+v_1e_{01}+u_1e_{20})+\sin\theta(w_1e_2-v_1e_{20}+u_1e_{01})\\
    r[\phi(e_1)]&=r[w_1e_1+v_1e_{01}+u_1e_{20}]\\
    &=w_1(\cos\theta e_1+\sin\theta e_2)+v_1(\cos\theta e_{01}-\sin\theta e_{20})+u_1(\sin\theta e_{01}+\cos\theta e_{20})\\[5pt]
    \phi(r[e_2])&=\phi(\cos\theta e_2-\sin\theta e_1)=\cos\theta(w_1e_2-v_1e_{20}+u_1e_{01})-\sin\theta(w_1e_1+v_1e_{01}+u_1e_{20})\\
    r[\phi(e_2)]&=r[w_1e_2-v_1e_{20}+u_1e_{01}]\\
    &=w_1(\cos\theta e_2-\sin\theta e_1)-v_1(\sin\theta e_{01}+\cos\theta e_{20})+u_1(\cos\theta e_{01}-\sin\theta e_{20})\\[5pt]
    \phi(r[e_{01}])&=\phi(\cos\theta e_{01}-\sin\theta e_{20})=w_2\cos\theta e_{01}-w_2\sin\theta e_{20}\\
    r[\phi(e_{01})]&=r[w_2e_{01}]=w_2(\cos\theta e_{01}-\sin\theta e_{20})\\[5pt]
    \phi(r[e_{20}])&=\phi(\sin\theta e_{01}+\cos\theta e_{20})=w_2\sin\theta e_{01}+w_2\cos\theta e_{20}\\
    r[\phi(e_{20})]&=r[w_2e_{20}]=w_2(\sin\theta e_{01}+\cos\theta e_{20})\\[5pt]
    \phi(r[e_{12}])&=\phi(e_{12})=w_2e_{12}+v_2e_{012}-u_2e_0\\
    r[\phi(e_{12})]&=r[w_2e_{12}+v_2e_{012}-u_2e_0]=w_2e_{12}+v_2e_{012}-u_2e_0\\[5pt]
    \phi(r[e_{012}])&=\phi(e_{012})=w_3e_{012}\\
    r[\phi(e_{012})]&=r[w_3e_{012}]=w_3e_{012}
\end{align*}
\end{proof}

\begin{proposition}
The geometric bilinear layer $\textrm{Geometric}(w,x,y,z)=\textrm{Concatenate}(wx, \textrm{Join}(y,z))$ is \SE2-equivariant.
\end{proposition}
\begin{proof}
Geometric product is equivariant as $(uwu^{-1})(uxu^{-1})=u(wx)u^{-1}$ for any $w,x$ and operator $u$.
That join is equivariant is proven in ~\citep{gatr}.
Alternatively, one can use linearity and check that $\textrm{Join}(uyu^{-1},uzu^{-1})=u\cdot\textrm{Join}(y,z)\cdot u^{-1}$ for any operator $u$ and \textit{basis} multivectors $y,z$.
\end{proof}

\begin{lemma}
If $f:\GA\to\GA$ is equivariant and $\rho:\GA\to\mathbb{R}$ is invariant, then $g(x):=\rho(x)f(x)$ is equivariant.
\end{lemma}
\begin{proof}
For any operator $u$ and multivector $x$,
\begin{align*}
    g(uxu^{-1})&=\rho(uxu^{-1})f(uxu^{-1})\\
    &=\rho(x)\cdot uf(x)u^{-1}\\
    &=ug(x)u^{-1}
\end{align*}
\end{proof}

\begin{proposition}
The scalar-gated activation $\textrm{GatedRELU}(x)=\textrm{RELU}(\langle x\rangle_0)x$ is \SE2-equivariant.
\end{proposition}
\begin{proof}
By Lemma 1 it suffices to show that $\langle x\rangle_0$ is \SE2-invariant. Write
\begin{equation*}
    x=x'+x_0e_0+x_1e_1+x_2e_2+x_{01}e_{01}+x_{20}e_{20}+x_{12}e_{12}+x_{012}e_{012}
\end{equation*}
If $u=1-\tfrac{a}{2}e_{01}+\tfrac{b}{2}e_{20}$ is a translation, we can compute
\begin{multline}
    u[x]=x'+(x_0-ax_1-bx_2)e_0+x_1e_1+x_2e_2\\
    +(x_{01}+bx_{12})e_{01}+(x_{20}+ax_{12})e_{20}+x_{12}e_{12}+x_{012}e_{012}
\end{multline}
hence $\langle u[x]\rangle_0=x'=\langle x\rangle_0$ so $\langle x\rangle_0$ is translation-invariant.
Similarly, if $u=\cos\frac{\theta}{2}-\sin\frac{\theta}{2}e_{12}$ is a rotation,
\begin{multline}
    u[x]=x'+x_0e_0+(x_1\cos\theta-x_2\sin\theta)e_1+(x_1\sin\theta+x_2\cos\theta)e_2\\
    +(x_{01}\cos\theta+x_{20}\sin\theta)e_{01}+(x_{20}\cos\theta-x_{01}\sin\theta)e_{20}+x_{12}e_{12}+x_{012}e_{012}
\end{multline}
hence $\langle x\rangle_0$ is rotation-invariant.
\end{proof}

\begin{lemma}
The inner product $\langle x,y\rangle:=x'y'+x_1y_1+x_2y_2+x_{12}y_{12}$ is \SE2-invariant.
\end{lemma}
\begin{proof}
If $u$ is a translation, by equation (21)
\begin{align*}
    \langle u[x],u[y]\rangle=x'y'+x_1y_1+x_2y_2+x_{12}y_{12}=\langle x,y\rangle
\end{align*}
If $u$ is a rotation, by equation (22)
\begin{align*}
    \langle u[x],u[y]\rangle&=x'y'+(x_1\cos\theta-x_2\sin\theta)(y_1\cos\theta-y_2\sin\theta)\\
    &\hspace{0.2in}+(x_1\sin\theta+x_2\cos\theta)(y_1\sin\theta+y_2\cos\theta)+x_{12}y_{12}\\
    &=x'y'+x_1y_1+x_2y_2+x_{12}y_{12}\\
    &=\langle x,y\rangle
\end{align*}
\end{proof}

\begin{proposition}
The normalization layer $\textrm{LayerNorm}(x) = x / \sqrt{ \mathbb{E}\langle x,x\rangle + \varepsilon}$ is \SE2-equivariant.
\end{proposition}
\begin{proof}
This follows from Lemma 1 and Lemma 2.
\end{proof}

\begin{proposition}
Multivector attention is \SE2-equivariant.
\end{proposition}
\begin{proof}
By Lemma 2, attention logits
\begin{equation*}
    \frac{\sum_{c=1}^C\langle q_c,k_c\rangle}{\sqrt{4C}}
\end{equation*}
(and hence attention scores) are invariant. The result follows from applying Lemma 1 and observing that a sum of equivariant multivector features is equivariant.
\end{proof}
\label{sec:equivariance_proofs}

\section{Additional qualitative results}

\begin{figure}[b]
\centering
\begin{subfigure}[b]{0.325\textwidth}
\centering
\includegraphics[trim={25cm 20cm 10cm 20cm}, clip,width=\textwidth]{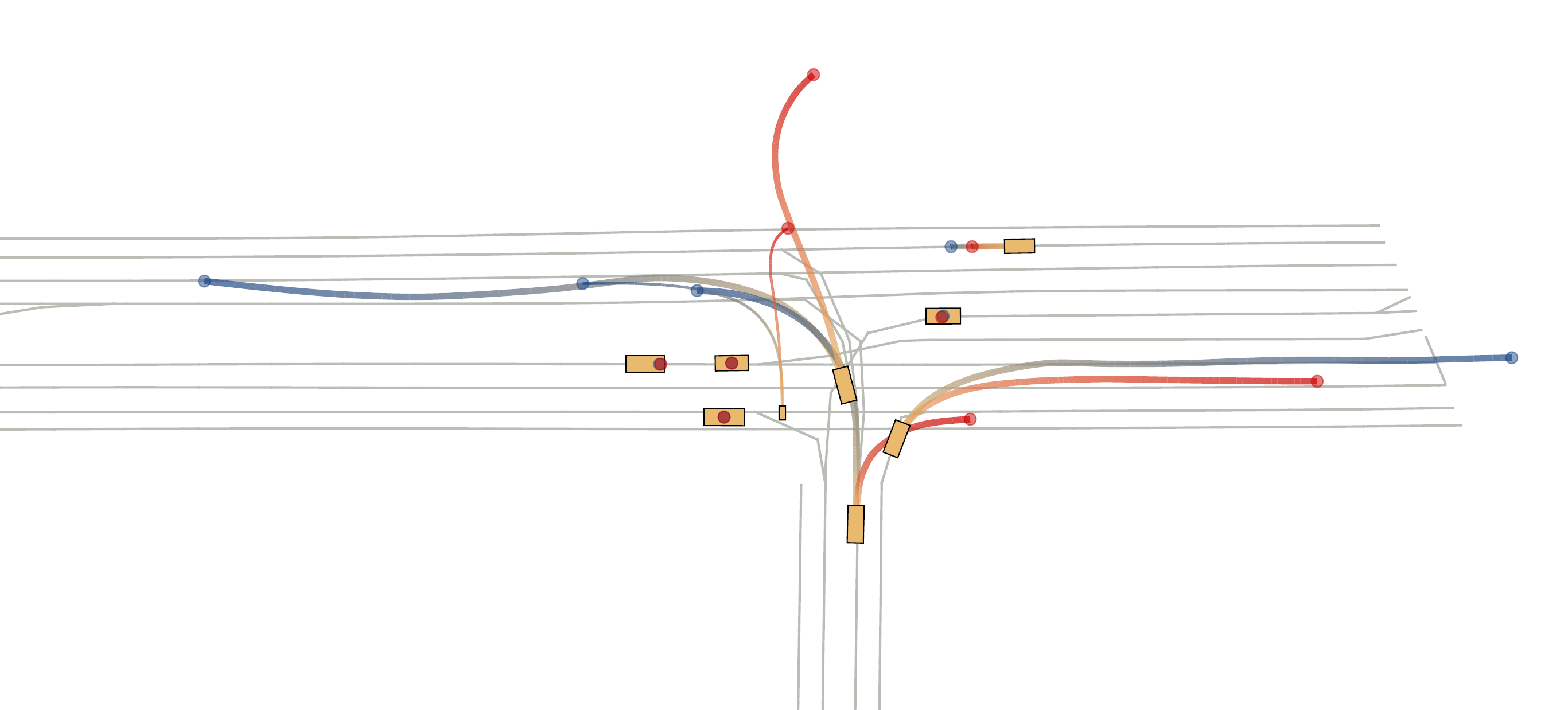}\\
\vspace{0.5cm}
\includegraphics[trim={10cm 20cm 20cm 25cm}, clip,width=\textwidth]{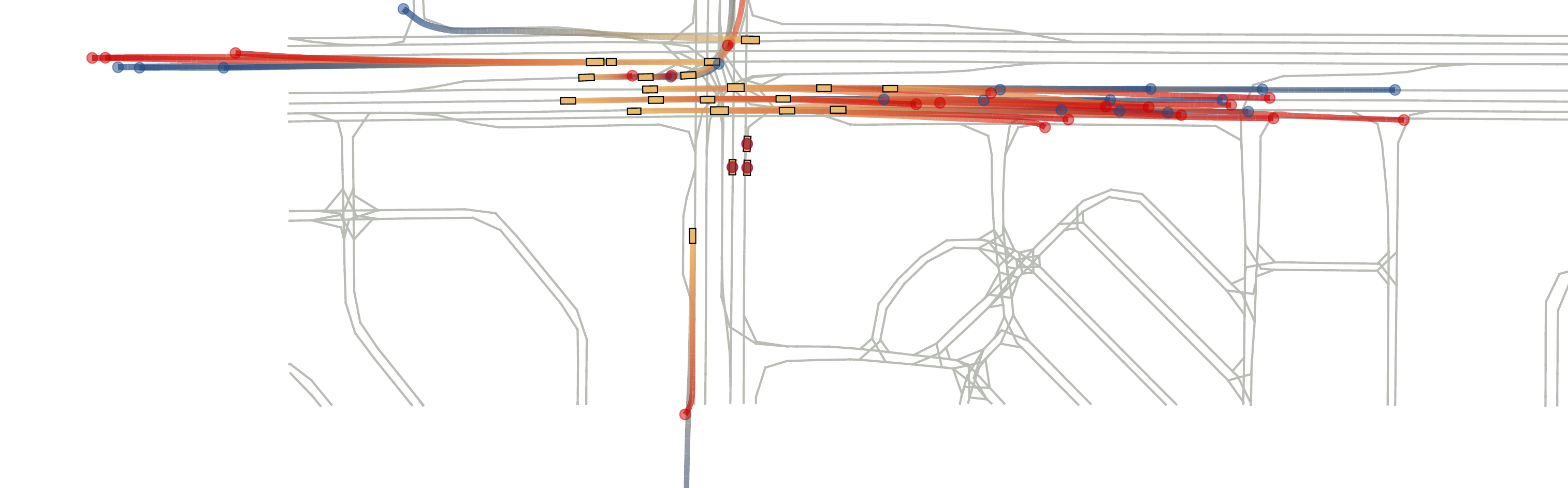}\\
\vspace{0.5cm}
\includegraphics[trim={10cm 20cm 20cm 20cm}, clip,width=\textwidth]{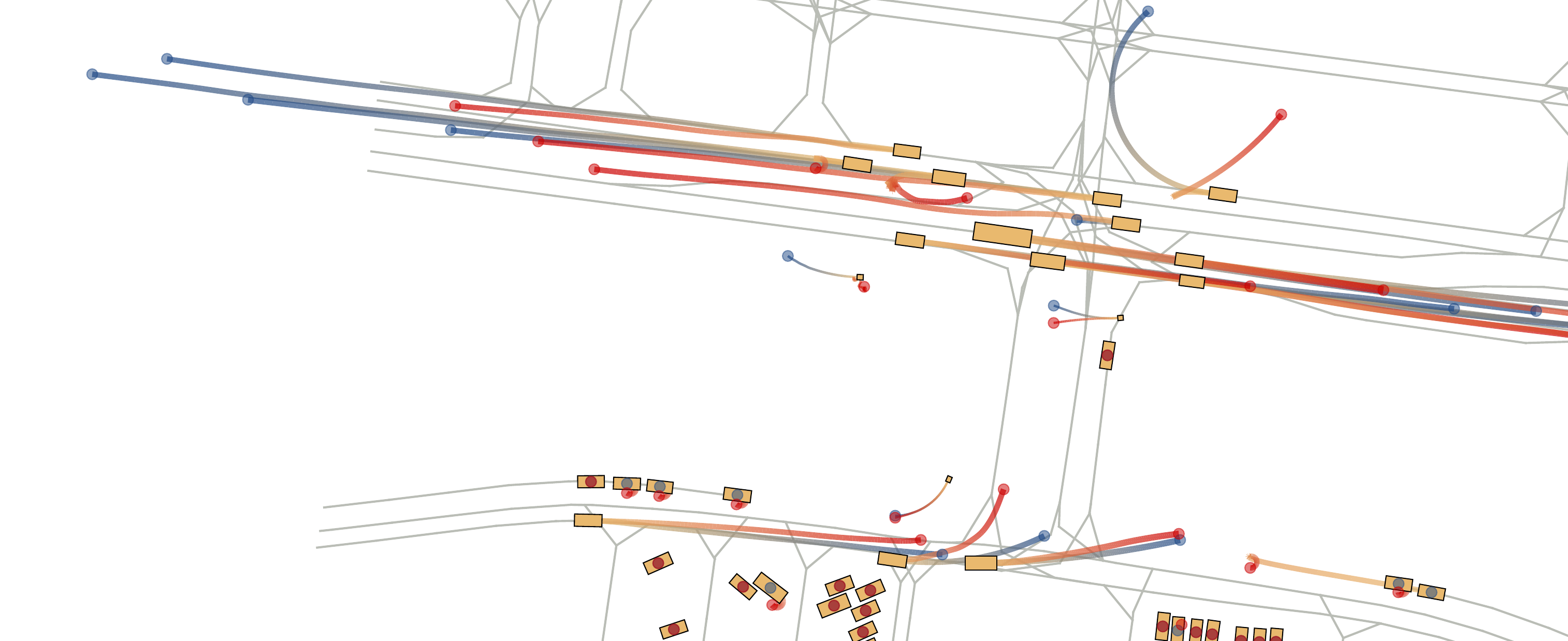}\\
\vspace{0.5cm}
\includegraphics[trim={10cm 20cm 20cm 20cm}, clip,width=\textwidth]{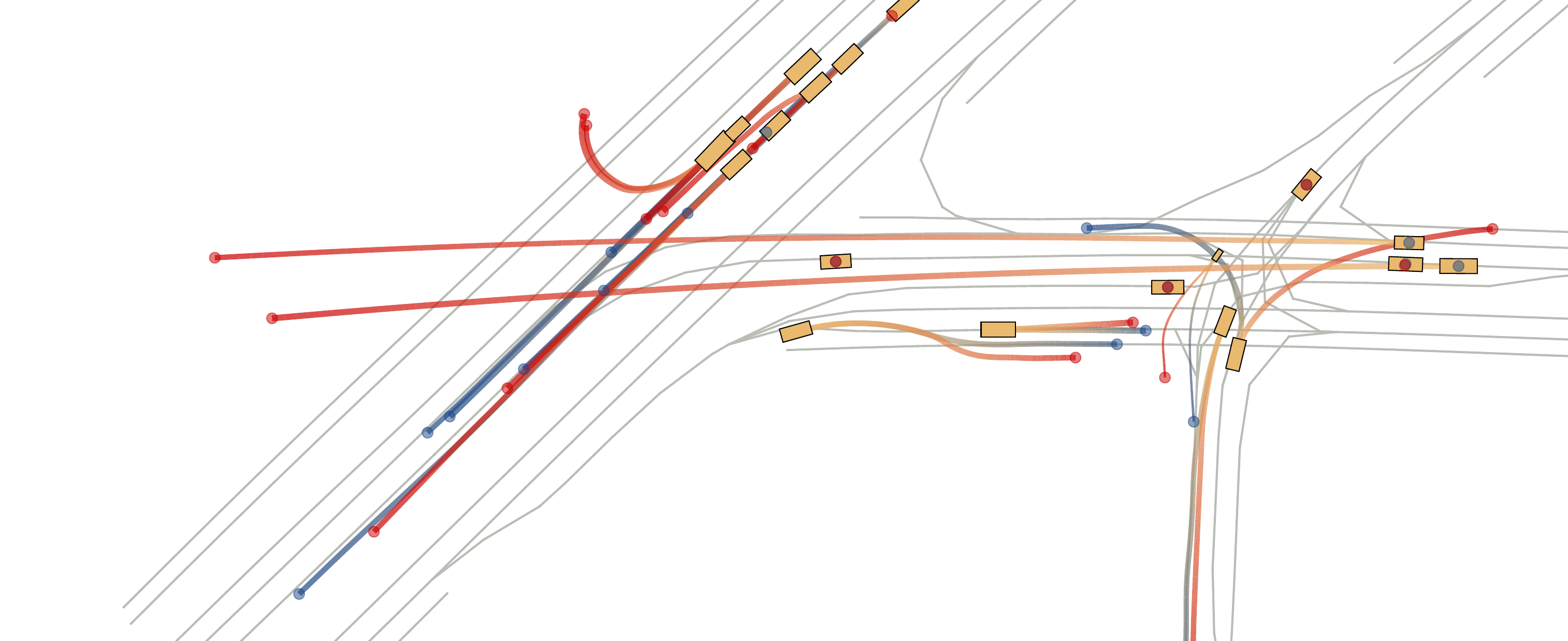}
\end{subfigure}
\begin{subfigure}[b]{0.325\textwidth}
\centering
\includegraphics[trim={25cm 20cm 10cm 20cm}, clip,width=\textwidth]{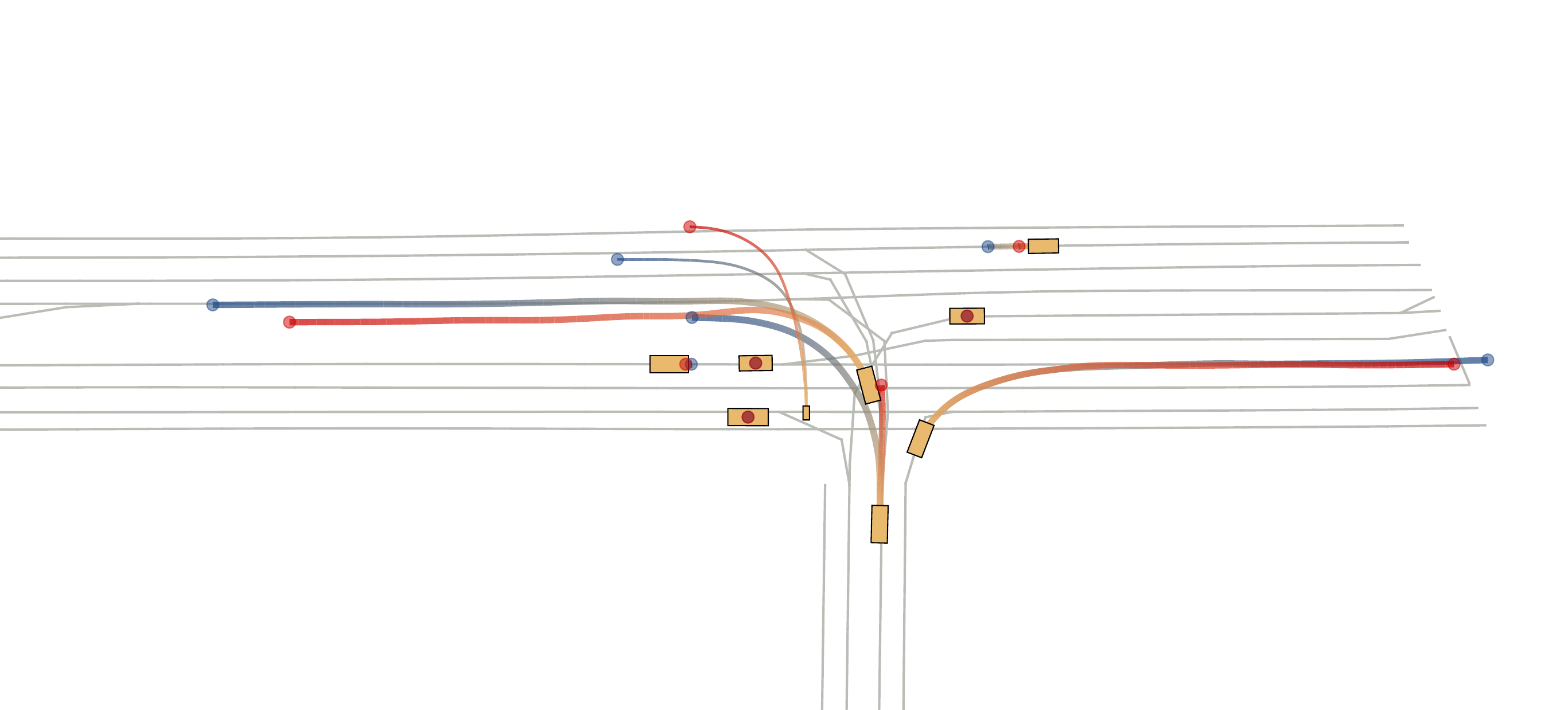}\\
\vspace{0.5cm}
\includegraphics[trim={10cm 20cm 20cm 25cm}, clip,width=\textwidth]{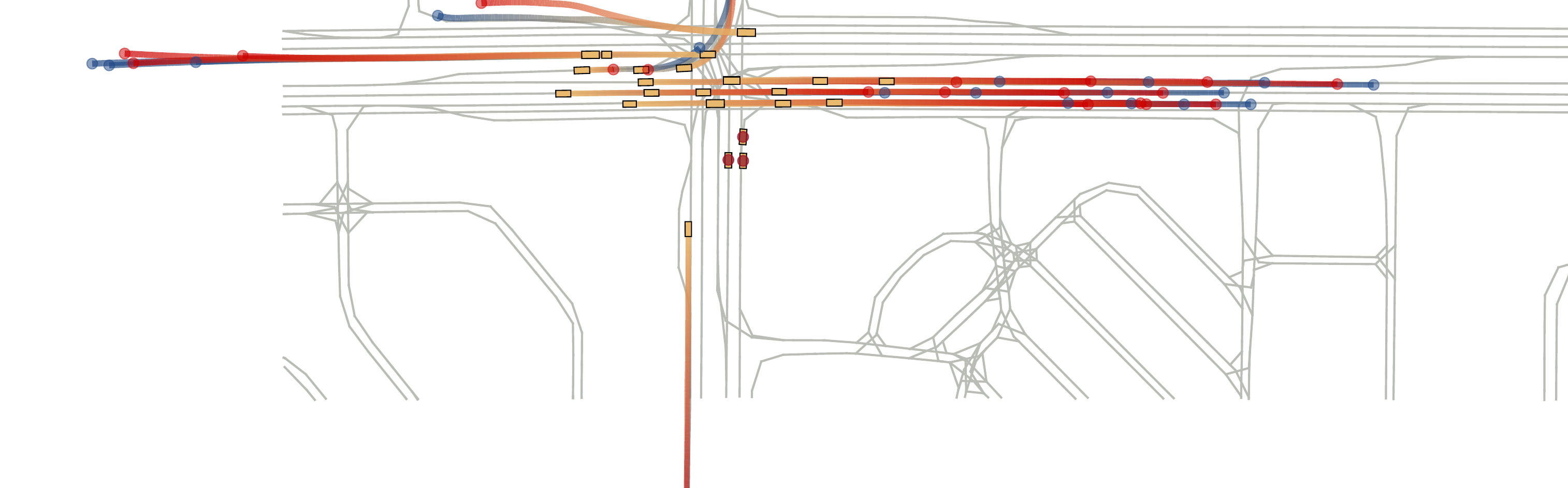}\\
\vspace{0.5cm}
\includegraphics[trim={10cm 20cm 20cm 20cm}, clip,width=\textwidth]{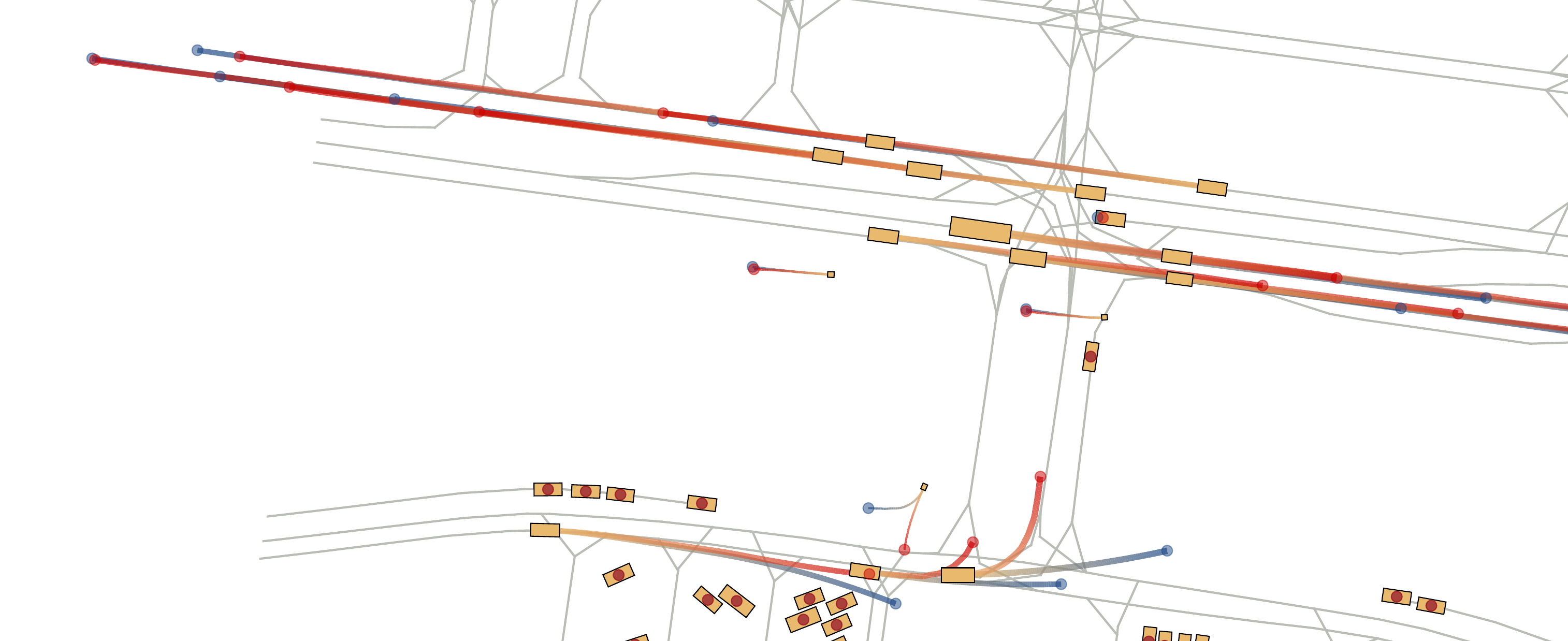}\\
\vspace{0.5cm}
\includegraphics[trim={10cm 20cm 20cm 20cm}, clip,width=\textwidth]{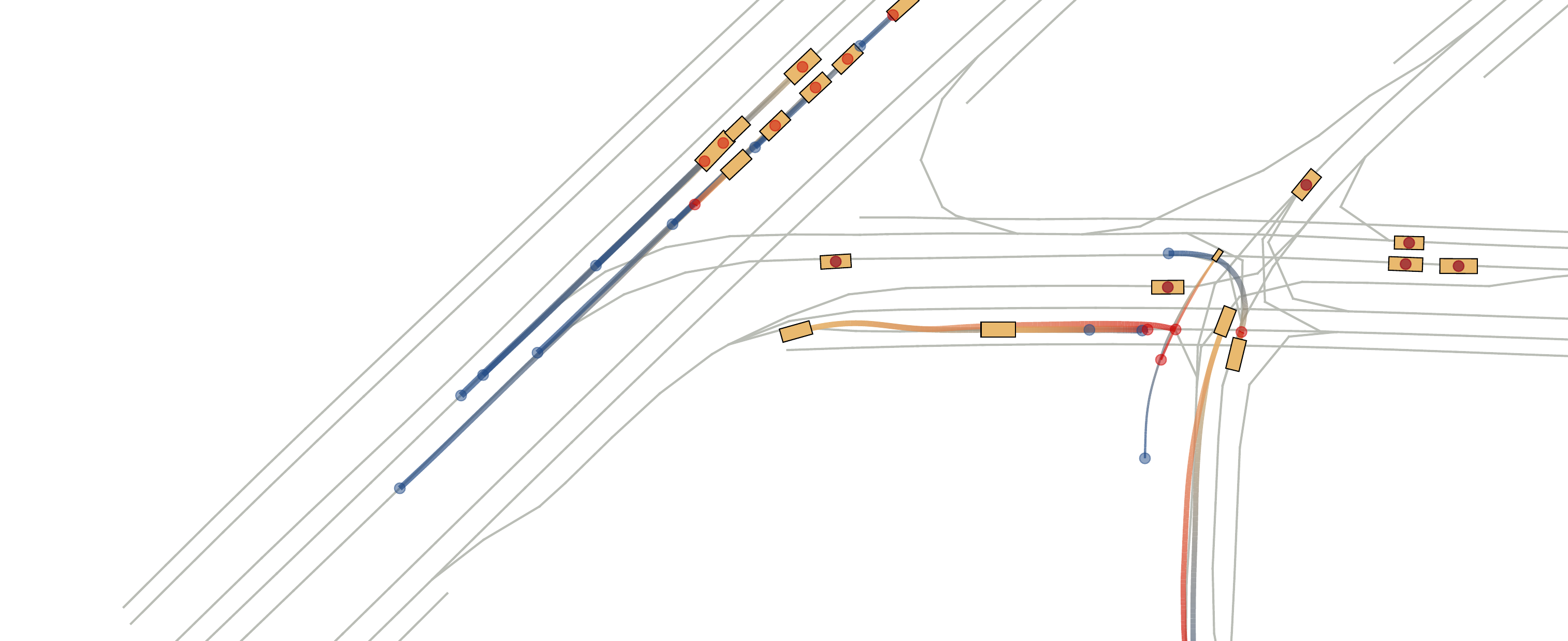}
\end{subfigure}
\begin{subfigure}[b]{0.325\textwidth}
\centering
\includegraphics[trim={25cm 20cm 10cm 20cm}, clip,width=\textwidth]{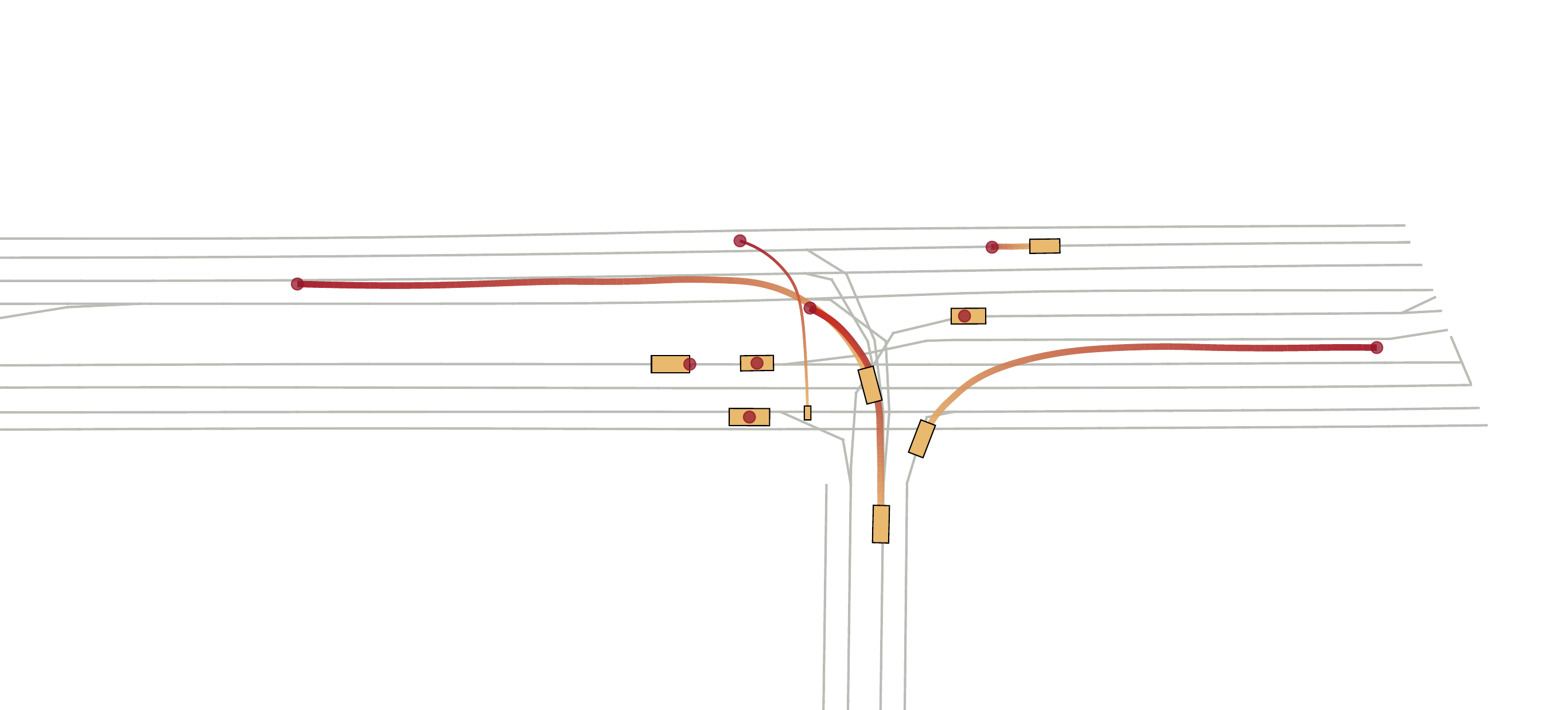}\\
\vspace{0.5cm}
\includegraphics[trim={10cm 20cm 20cm 25cm}, clip,width=\textwidth]{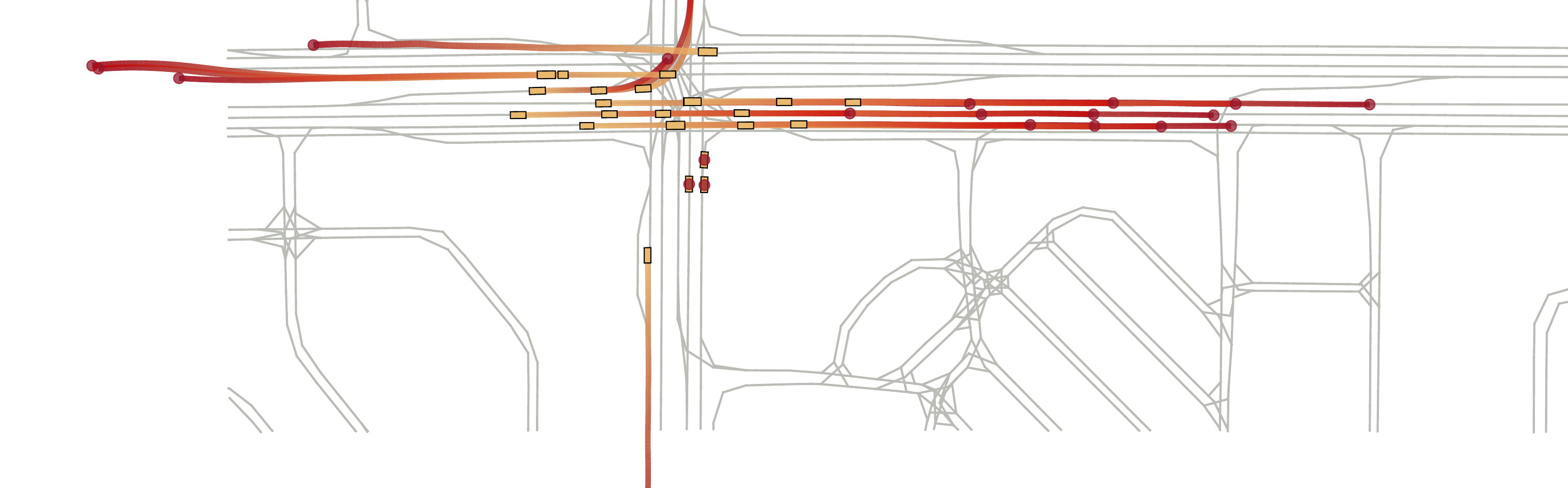}\\
\vspace{0.5cm}
\includegraphics[trim={10cm 20cm 20cm 20cm}, clip,width=\textwidth]{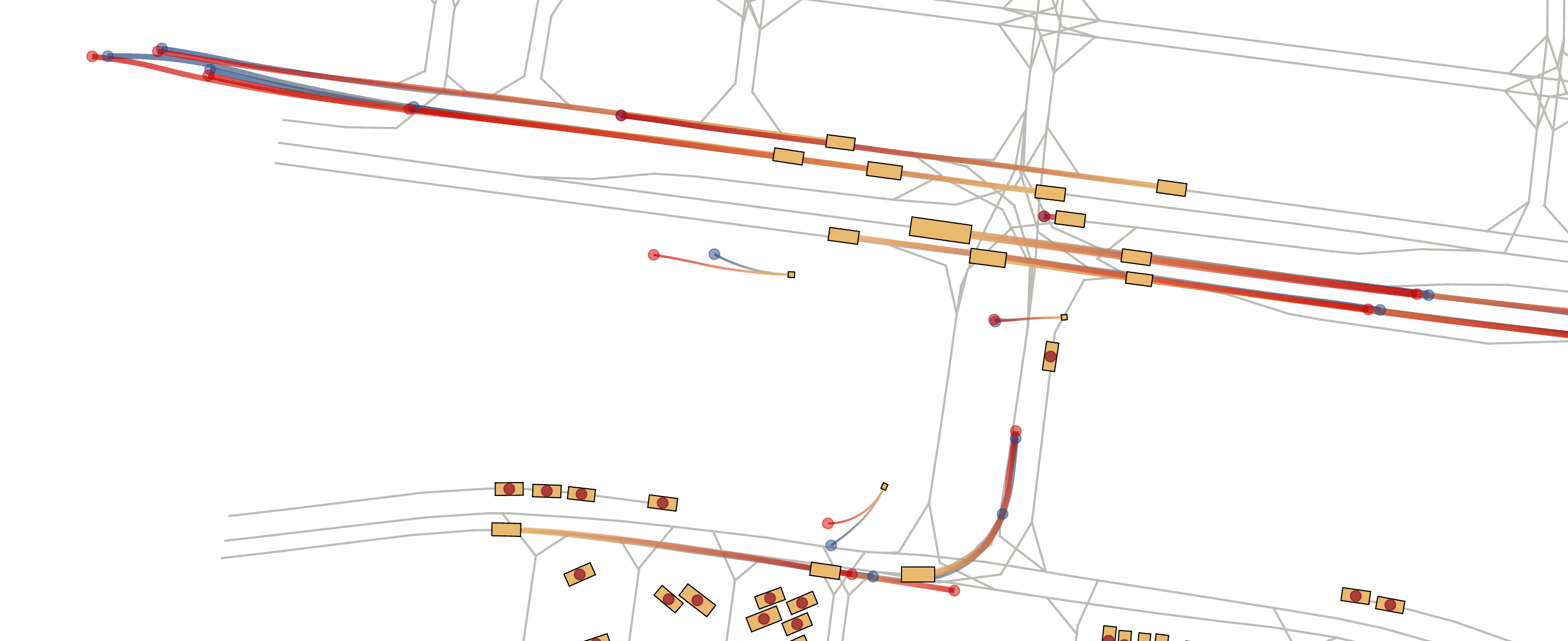}\\
\vspace{0.5cm}
\includegraphics[trim={10cm 20cm 20cm 20cm}, clip,width=\textwidth]{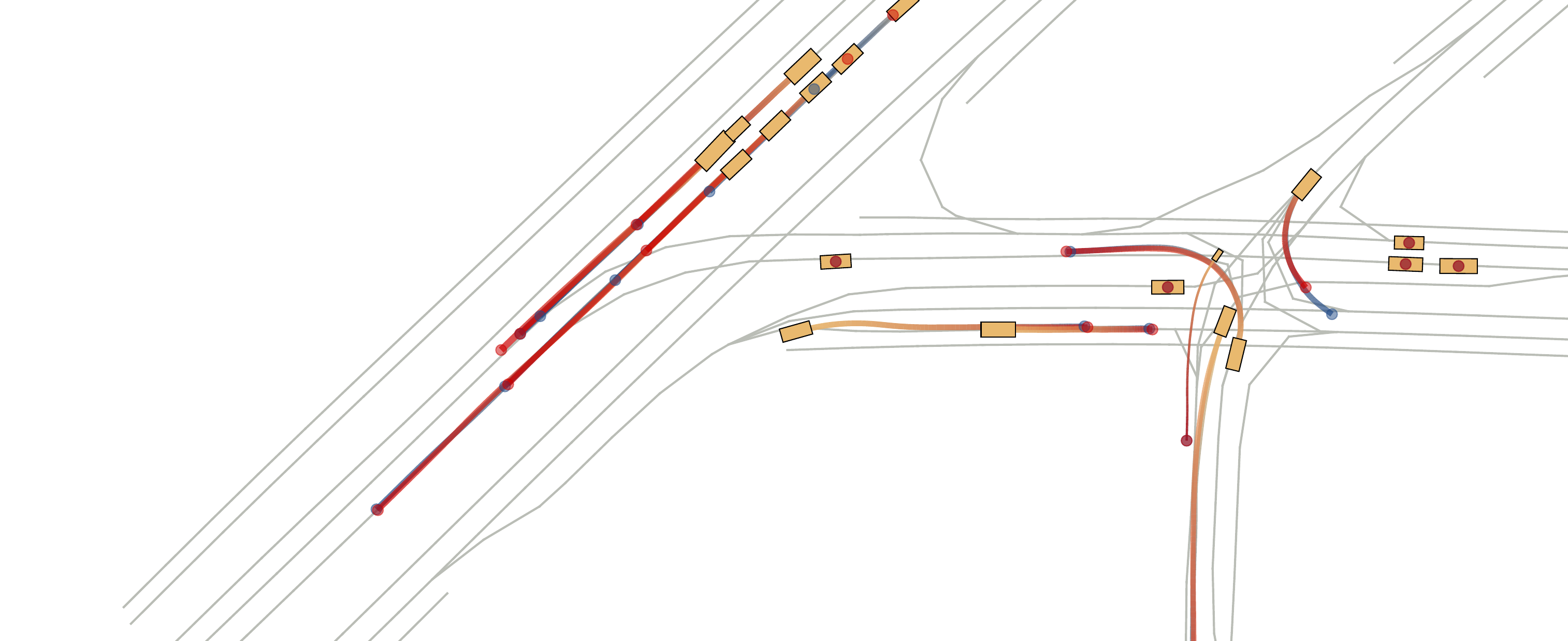}
\end{subfigure}
\caption{
\textbf{Robustness to roto-translations.}
Additional visualizations showing Transformer (\textbf{left}), Transformer + DRoPE (\textbf{middle}), and \ourmodel (\textbf{right})'s
robustness to roto-translations.
In each figure, we overlay rollouts from the original coordinate frame
vs one rotated by 90\textdegree and translated by 100m forward.
Blue trajectories visualize model predictions in the original input, and red visualize predictions in the transformed scene.
}
\label{fig:robustness}
\end{figure}

\begin{figure}[t]
\centering
\begin{subfigure}[b]{0.485\textwidth}
\centering
\includegraphics[trim={20cm 15cm 18cm 25cm}, clip,width=\textwidth]{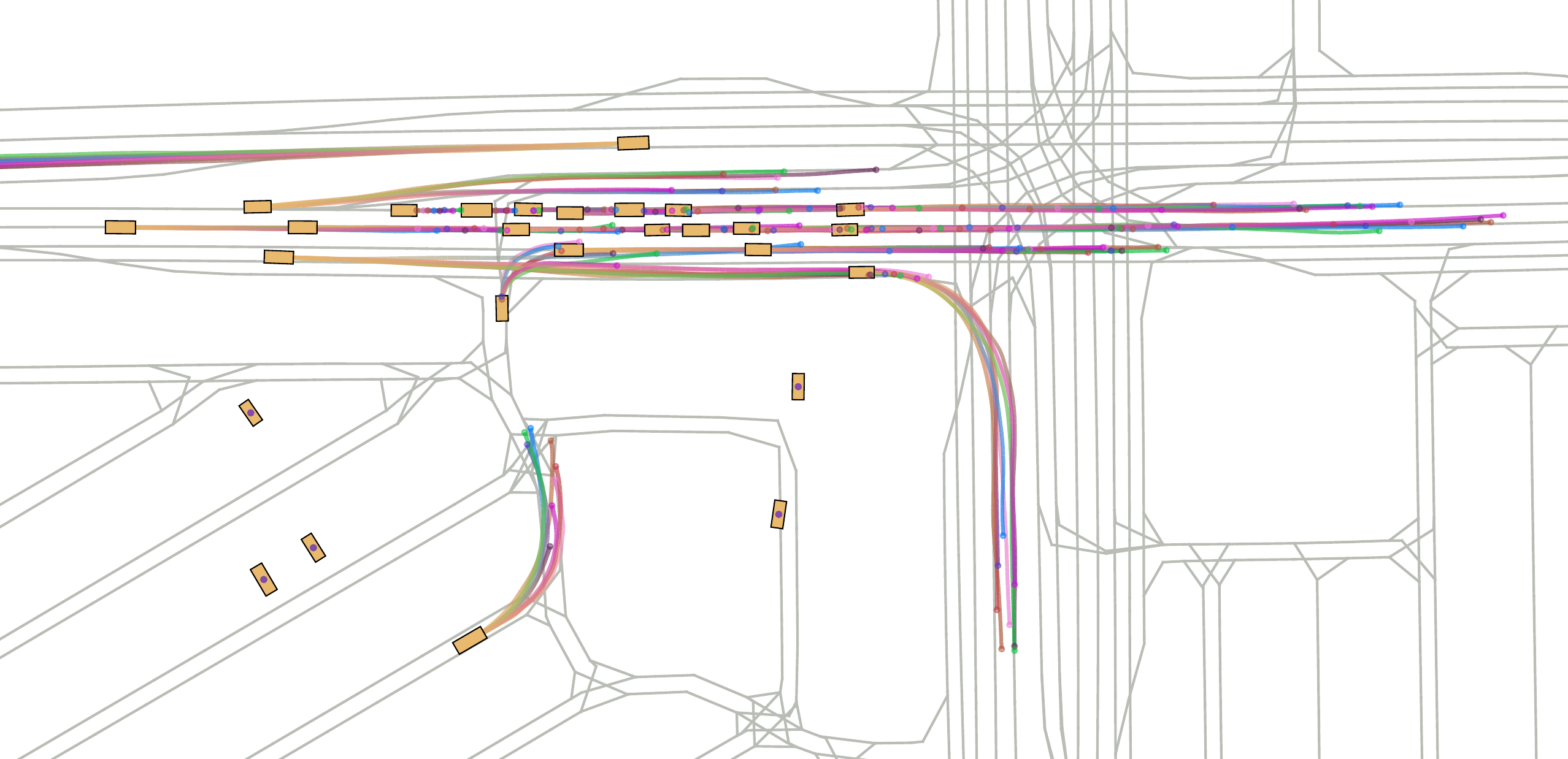}\\
\includegraphics[trim={22cm 22cm 15cm 18cm}, clip,width=\textwidth]{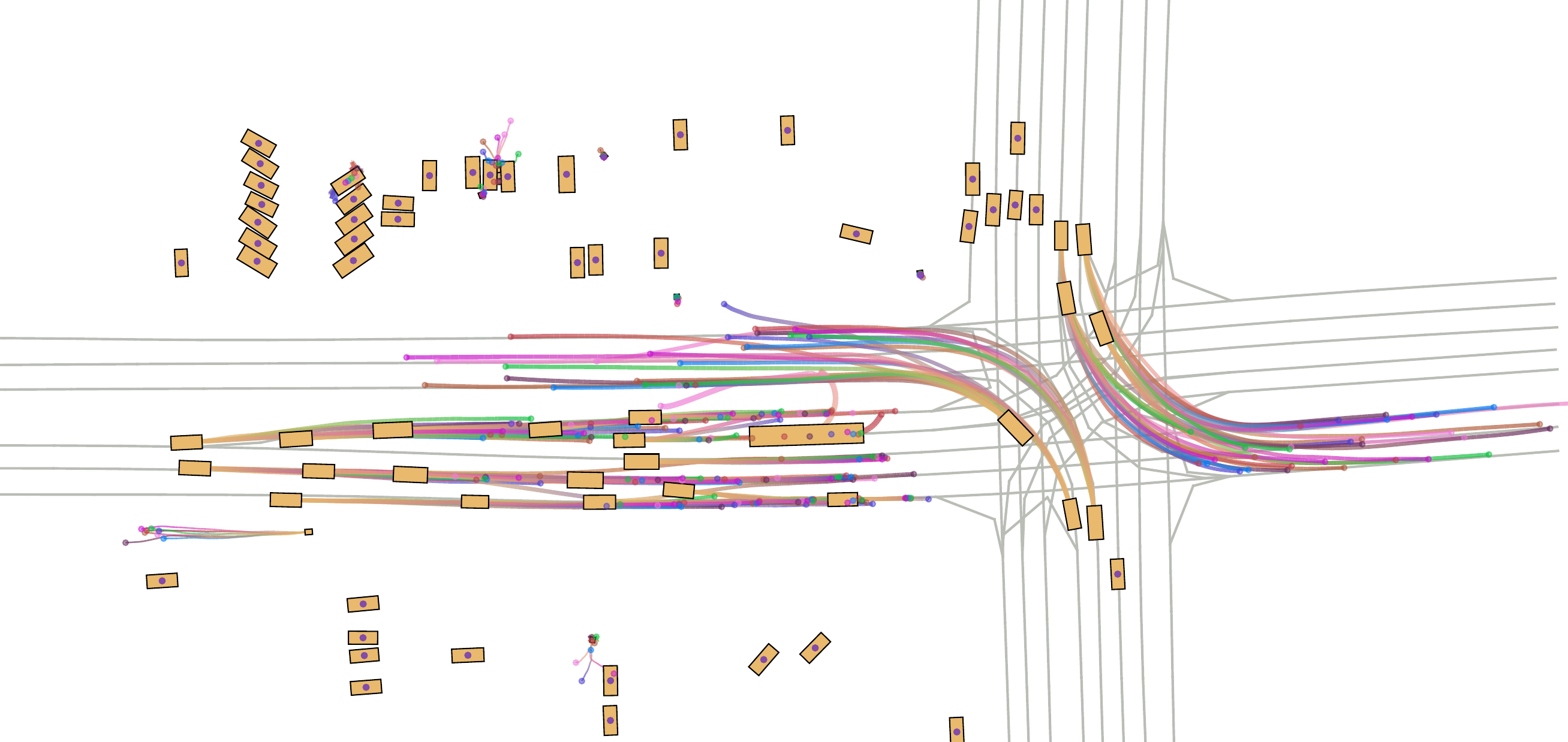}\\
\includegraphics[trim={22cm 10cm 15cm 10cm}, clip,width=\textwidth]{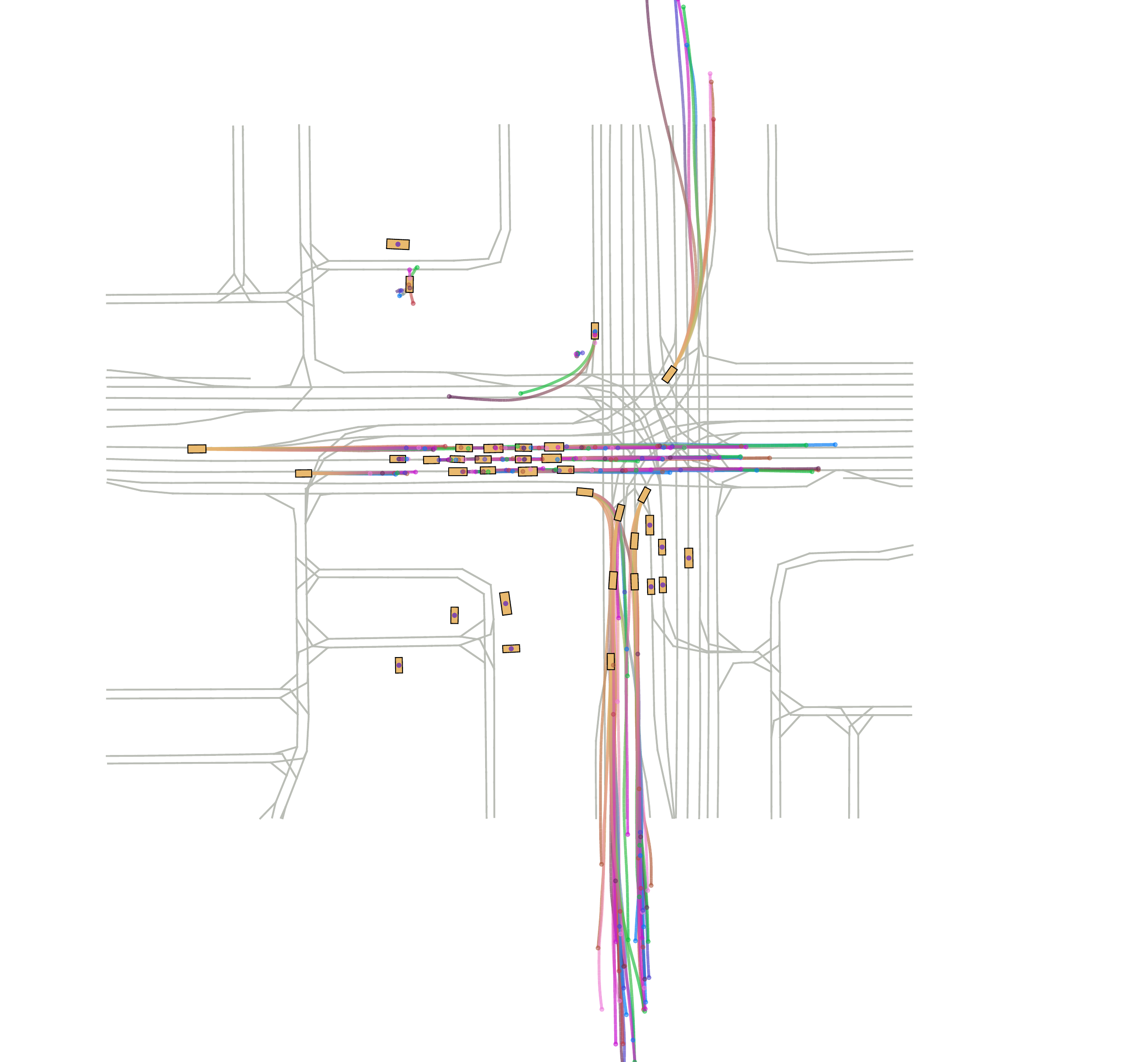}\\
\includegraphics[trim={22cm 15cm 15cm 15cm}, clip,width=\textwidth]{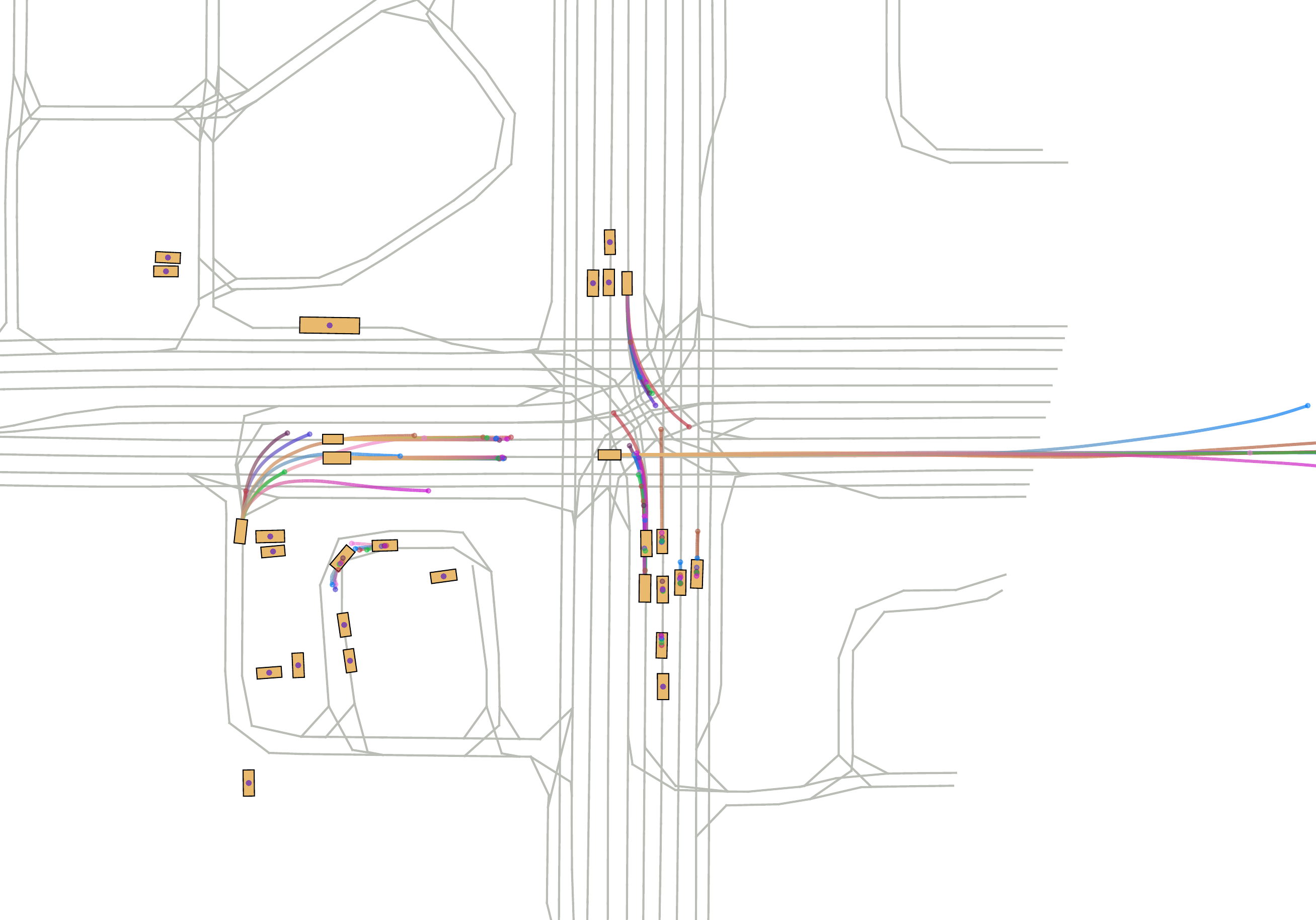}
\end{subfigure}
\begin{subfigure}[b]{0.485\textwidth}
\centering
\includegraphics[trim={25cm 15cm 15cm 20cm}, clip,width=\textwidth]{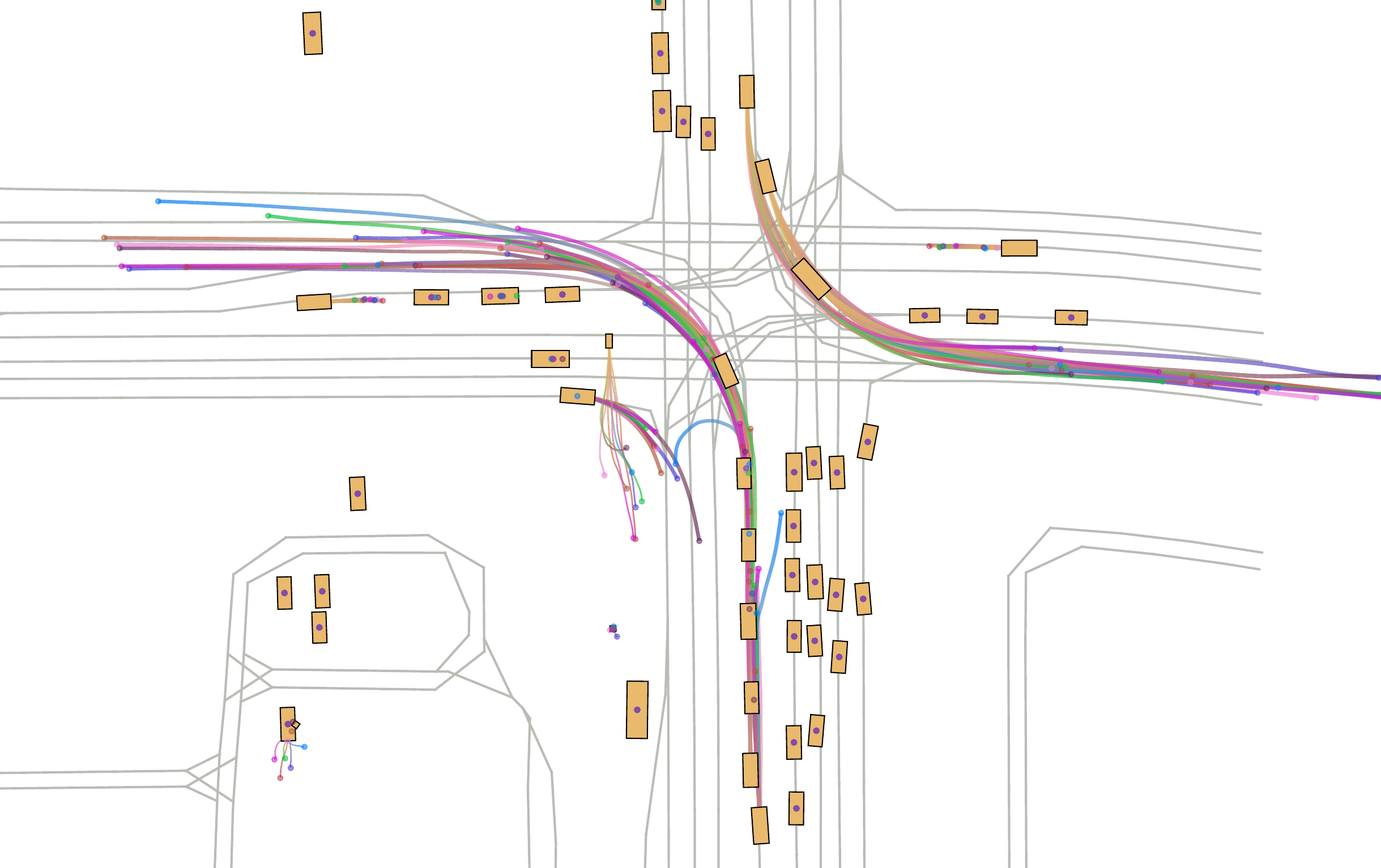}\\
\includegraphics[trim={25cm 15cm 12cm 15cm}, clip,width=\textwidth]{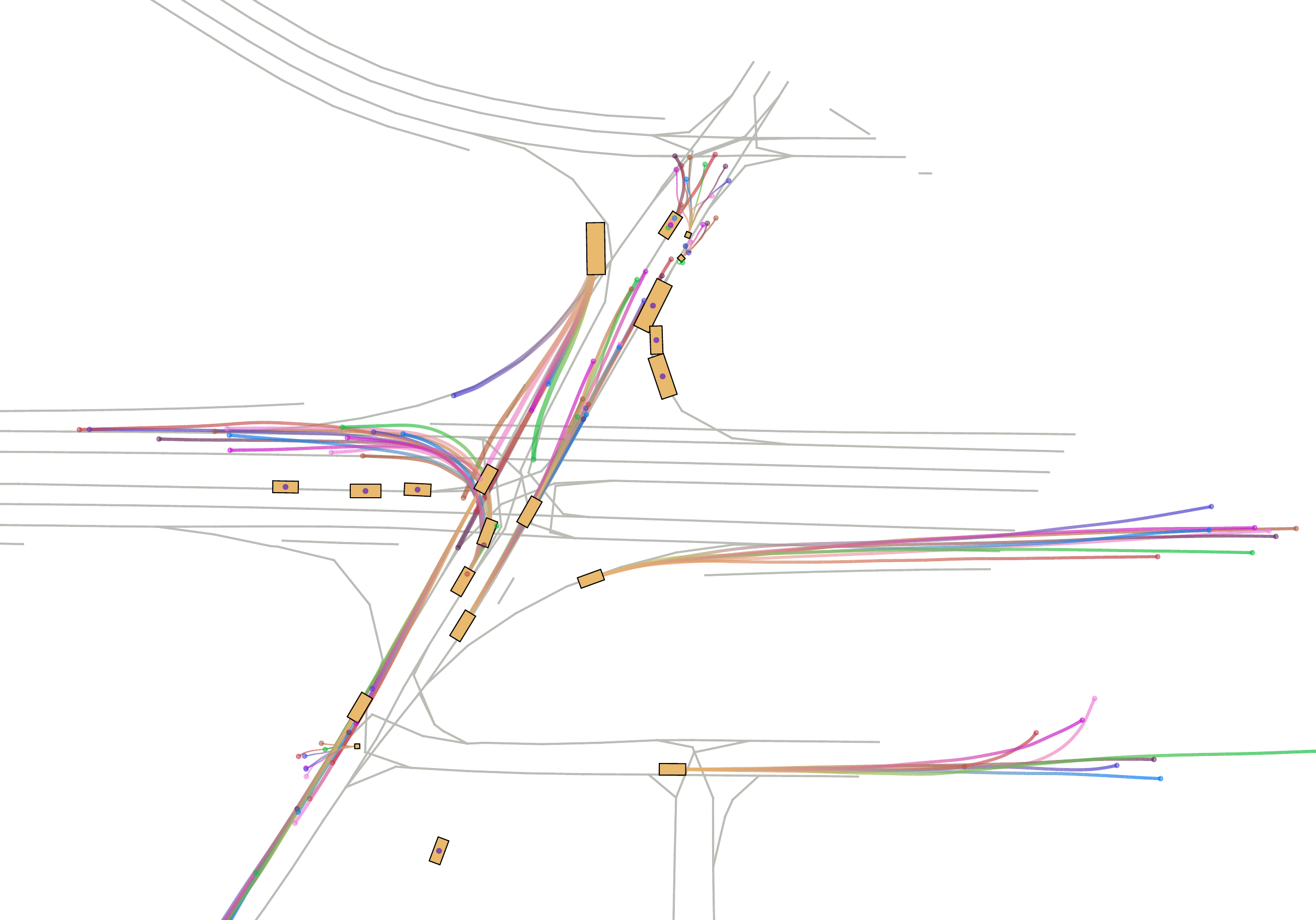}\\
\includegraphics[trim={28cm 15cm 15cm 25cm}, clip,width=\textwidth]{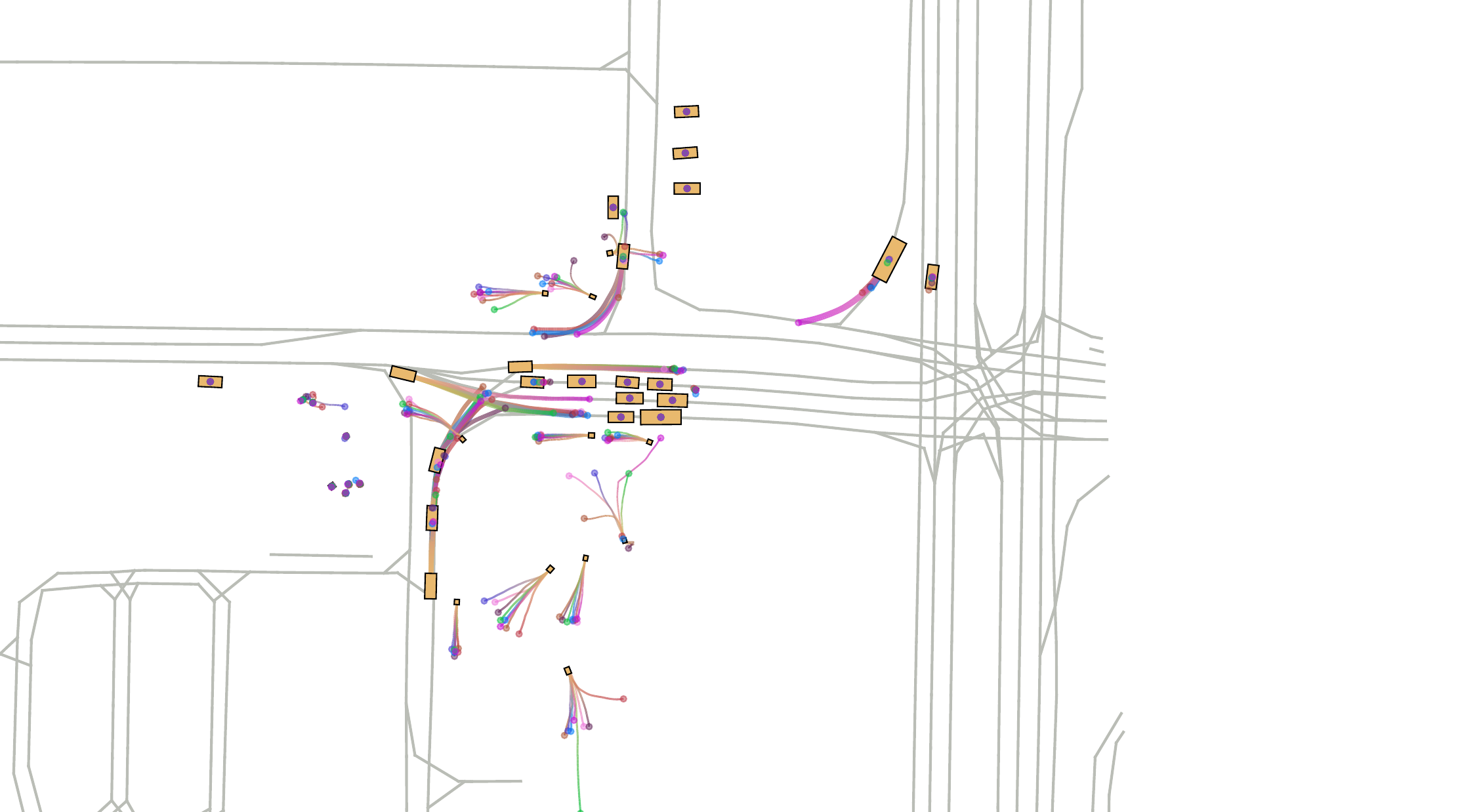}\\
\includegraphics[trim={28cm 25cm 20cm 25cm}, clip,width=\textwidth]{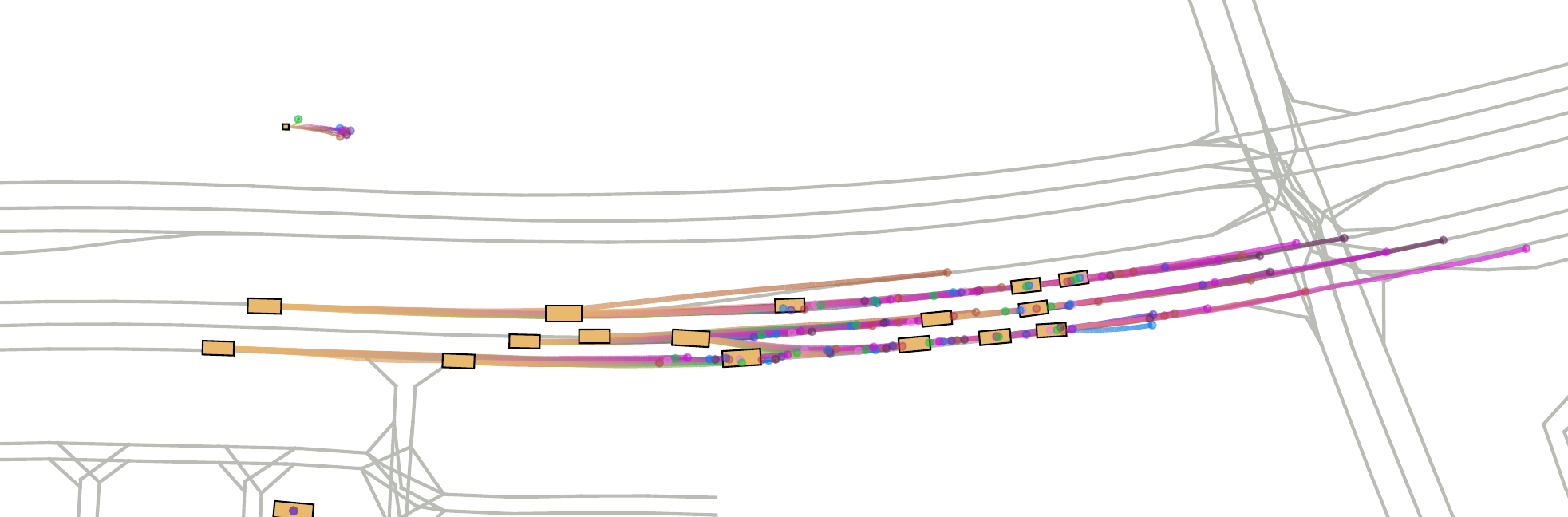}\\
\includegraphics[trim={28cm 25cm 20cm 25cm}, clip,width=\textwidth]{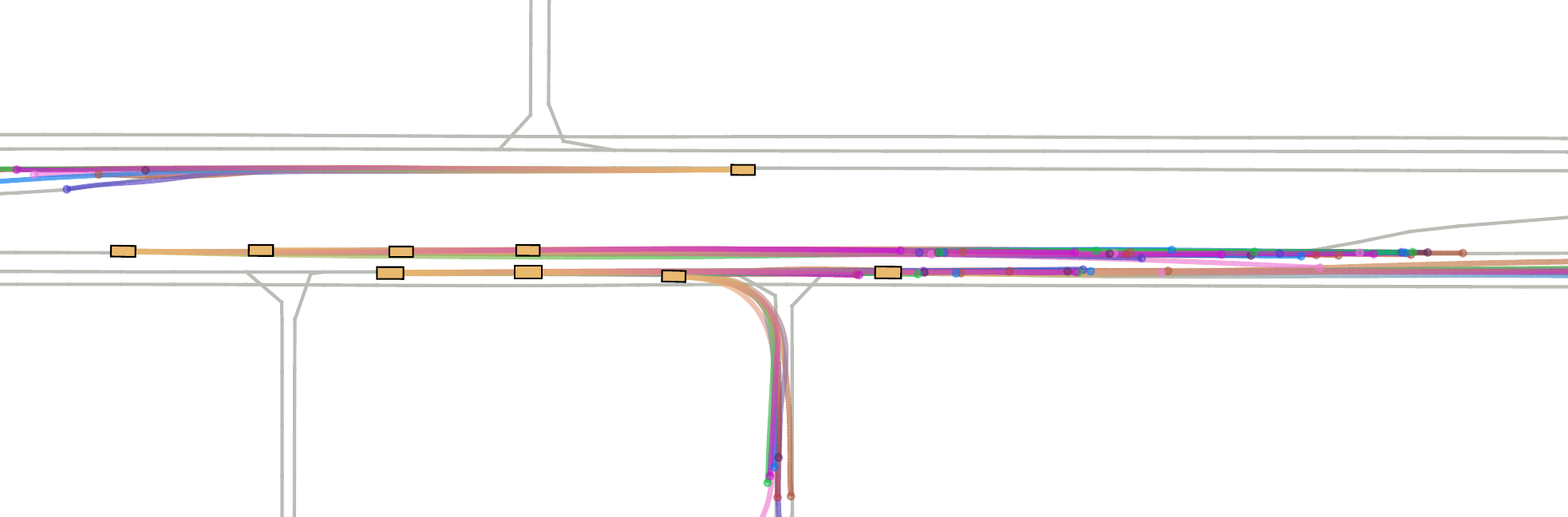}
\end{subfigure}
\caption{Visualizations of \textbf{multi-modal} \ourmodel rollouts. 
In each figure, we overlay trajectories from multiple ($k=8$) rollouts.
Different actors of the same color corresponds to the same rollout.}
\label{fig:multimodal}
\end{figure}

\reffig{fig:robustness} provides additional comparisons on models' robustness to roto-translations.
\reffig{fig:multimodal} shows the overlayed multi-modal trajectories of \ourmodel's rollouts.

\clearpage

\section{Additional details}
\subsection{Architecture details}

In this section, we provide algorithmic details of the \ourmodel architecture.
Multivector features are assumed to have shape $[\textrm{batch}\times\textrm{num\_tokens}\times\textrm{mv\_channels}\times8]$ where the last axis corresponds to the 8 components of each multivector.
Scalar features have shape $[\textrm{batch}\times\textrm{num\_tokens}\times\textrm{channels}]$.

\paragraph{Equivariant MLP.} Throughout the model, when multivector features pass through an equivariant layer, scalar features pass through the standard version of that layer.
Our equivariant MLP is analagous to a standard MLP, with the addition of the geometric bilinears.

\begin{algorithmic}[1]
\Function{GeometricBilinear}{w, x, y, z}
  \State gp $\gets$ batch\_geometric\_product(w, x)
  \State join $\gets$ dual(batch\_wedge\_product(dual(y), dual(z))) \\

  \LComment{$[\textrm{batch}\times\textrm{num\_tokens}\times(2*\textrm{mv\_channels})\times8]$}
  \State \Return concatenate\_channels(gp, join)
\EndFunction \\

\Function{EquivariantMLP}{multivector\_features, scalar\_features}
  \LComment{normalization}
  \State mv\_features $\gets$ EquivariantLayerNorm(multivector\_features)
  \State features $\gets$ LayerNorm(scalar\_features) \\

  \LComment{geometric bilinears}
  \State mv\_features $\gets$ EquivariantLinear(mv\_features)
  \State w, x, y, z $\gets$ mv\_features.split(num\_chunks = 4, dim = -2)
  \State mv\_features $\gets$ GeometricBilinear(w, x, y, z) \\

  \LComment{standard MLP}
  \State mv\_features $\gets$ EquivariantLinear(mv\_features)
  \State features $\gets$ Linear(features)
  \State mv\_features $\gets$ GatedRELU(mv\_features)
  \State features $\gets$ RELU(features)
  \State mv\_features $\gets$ EquivariantLinear(mv\_features)
  \State features $\gets$ Linear(features) \\

  \LComment{residual connection}
  \State \Return mv\_features + multivector\_features, features + scalar\_features
\EndFunction
\end{algorithmic}

\paragraph{Attention.} We show how to compute multivector attention, where the logits are
\begin{equation*}
    \frac{\sum_{c=1}^C\langle q_c,k_c\rangle+\sum_{c=1}^{C'}q^s_c k^s_c}{\sqrt{4C+C'}}.
\end{equation*}
Observe that a sum of dot products can be equivalently expressed as a longer dot product, e.g. for $u_1,v_1,u_2,v_2\in\mathbb{R}^n$, $u_1\cdot v_1+u_2\cdot v_2=\mathbf{u}\cdot\mathbf{v}$ where $\mathbf{u}=u_1|u_2,\mathbf{v}=v_1|v_2\in\mathbb{R}^{2n}$ are concatenations.
Therefore, we can use standard dot-product attention after concatenating the appropriate multivector components with the invariant scalars along the embedding dimension.
To incorporate distance awareness, one just needs to further concatenate the components of the distance awareness maps $\phi(\cdot)$ and $\psi(\cdot)$ to the queries and keys, respectively.

\begin{algorithmic}[1]
\Function{EquivariantAttention}{mv\_query, mv\_key, mv\_value, query, key, value, mask}
  \LComment{extract components for invariant inner product $\langle x,y\rangle=x_sy_s+x_1y_1+x_2y_2+x_{12}y_{12}$}
  \State $q_s, \_, q_1, q_2, \_, \_, q_{12}, \_ \gets$ mv\_query.split(dim = -1)
  \State $k_s, \_, k_1, k_2, \_, \_, k_{12}, \_ \gets$ mv\_key.split(dim = -1) \\

  \LComment{$[\textrm{batch}\times\textrm{num\_queries}\times(4*\textrm{mv\_channels}+\textrm{channels})]$}
  \State flattened\_q $\gets$ concatenate\_channels($q_s, q_1, q_2, q_{12}$, query) \\

  \LComment{$[\textrm{batch}\times\textrm{num\_keys}\times(4*\textrm{mv\_channels}+\textrm{channels})]$}
  \State flattened\_k $\gets$ concatenate\_channels($k_s, k_1, k_2, k_{12}$, key) \\

  \LComment{$[\textrm{batch}\times\textrm{num\_keys}\times(8*\textrm{mv\_channels}+\textrm{channels})]$}
  \State flattened\_v $\gets$ concatenate\_mv\_components\_and\_scalars(mv\_value, value) \\

  \LComment{$[\textrm{batch}\times\textrm{num\_queries}\times(8*\textrm{mv\_channels}+\textrm{channels})]$}
  \State flattened\_output $\gets$ scaled\_dot\_product\_attention(flattened\_q, flattened\_k, flattened\_v, mask) \\

  \LComment{$[\textrm{batch}\times\textrm{num\_queries}\times\textrm{mv\_channels}\times8]$, $[\textrm{batch}\times\textrm{num\_queries}\times\textrm{channels}]$}
  \State \Return unconcatenate\_mv\_components\_and\_scalars(flattened\_output)
\EndFunction \\

\Function{Attention}{mv\_q\_features, q\_features, mv\_kv\_features, kv\_features, mask}
  \LComment{normalization}
  \State mv\_q\_features $\gets$ EquivariantLayerNorm(mv\_q\_features)
  \State mv\_kv\_features $\gets$ EquivariantLayerNorm(mv\_kv\_features)
  \State q\_features $\gets$ LayerNorm(q\_features)
  \State kv\_features $\gets$ LayerNorm(kv\_features) \\

  \LComment{compute query, key, value tensors}
  \State mv\_query $\gets$ EquivariantLinear(mv\_q\_features)
  \State mv\_key $\gets$ EquivariantLinear(mv\_kv\_features)
  \State mv\_value $\gets$ EquivariantLinear(mv\_kv\_features)
  \State query $\gets$ Linear(q\_features)
  \State key $\gets$ Linear(kv\_features)
  \State value $\gets$ Linear(kv\_features) \\

  \State mv\_features, features $\gets$ EquivariantAttention(mv\_query, mv\_key, mv\_value, query, key, value, mask) \\

  \LComment{residual connection}
  \State \Return mv\_features + mv\_q\_features, features + q\_features
\EndFunction
\end{algorithmic}

\paragraph{\ourmodel.} Each \ourmodel block consists of three factorized attention layers followed by an equivariant MLP and an invariant adapter.
In the invariant adapter, we compute an operator for each actor which converts global poses into that agent's coordinate frame.
Specifically, if an actor has position $(x,y)$ and heading $\theta$, then the operator for that actor consists of a translation by $(-x,-y)$ followed by a rotation of angle $-\theta$.

\begin{algorithmic}[1]
\Function{\ourmodel}{mv\_features, features, mv\_map\_features, map\_features}
  \For{$i\gets1\ldots N$}
    \LComment{actor-to-map cross-attention}
    \State mv\_features, features $\gets$ Attention(mv\_features, features, mv\_map\_features, map\_features) \\

    \LComment{actor-to-actor self-attention (batched across timesteps)}
    \State mv\_features, features $\gets$ Attention(mv\_features, features, mv\_features, features) \\

    \LComment{actor-to-time causal self-attention (batched across actors)}
    \State mv\_features, features $\gets$ Attention(mv\_features, features, mv\_features, features, CAUSAL\_MASK) \\

    \LComment{MLP}
    \State mv\_features, features $\gets$ EquivariantMLP(mv\_features, features) \\

    \LComment{invariant adapter}
    \State $u\gets$ operator for each actor \textit{// $[\textrm{batch}\times\textrm{num\_actors}\times8]$, duplicated across channels}
    \State features $\gets$ MLP(flatten\_components($u*\textrm{mv\_features}*u^{-1}$)) + features
  \EndFor \\

  \LComment{output invariant features only}
  \State \Return features
\EndFunction

\end{algorithmic}

\subsection{Baselines}
We provide implementation details for the Transformer, Transformer + DRoPE, and Transformer + RPE baselines.
All three baselines are edited from the original \ourmodel architecture to minimalize differences other than the ablated module.
In particular, we first set the multi-vector dimension to $0$, effectively make \ourmodel a standard transformer without positional encodings.
We next add in changes specific to each baseline, detailed below.

\subsection{Transformer}
In addition the invariant features of each agent,
we additonally include agent poses ($x, y, \theta$) in the global coordinate frame to the inputs of the initial MLP that encodes the auxiliary scalars.

\subsection{Transformer + DRoPE}
At the time of submission, there is no publicly released codes nor implementation details of DRoPE~\citet{drope}.
We hence faithfully reproduce it in our setup.
In particular, we use the ``intra-head'' formulation, splitting the positional encodings of $x, y$ and rotational $\theta$ separately by attention heads.
For the 2D translational part of the rotary encodings, we use a learnable, mixed frequency setup~\citet{ropevit}.
Notably, we found it neccessary to uplift the dimention of query and key by a factor of $4$ before attention -- otherwise, each head only has $4$ dimensions to be rotary-encoded which is too small.

\subsection{Transformer + RPE}
At each attention block, we additionally use an MLP to model the pairwise differences.
We use the GoRELA~\citet{gorela} encoder to encode the pairwise pose features, and broad-cast add them to the keys and values before applying attention:

\begin{align}
    k^{ij}_{rpe}, v^{ij}_{rpe} &= \text{GoRELA}(x^i, y^i, \theta^i, x^j, y^j, \theta^j) \\
    k &\leftarrow k + k_{rpe} \\
    v &\leftarrow v + v_{rpe}
\end{align}

\section{Latency}
\begin{table}[h]
\centering
\begin{tabular}{l c c c c}
\toprule
Method & \makecell{Training latency} & \makecell{Training peak memory} & \makecell{Inference latency} & \makecell{\SE2 equivariant?} \\
\cmidrule(lr){1-1}\cmidrule(lr){2-5}
Transformer & 196.86 ms & 3.28 GB & 42.64 ms & $\color{red}\times$ \\
Transformer + RPE & 328.70 ms & 15.19 GB & 82.48 ms & $\color{green}\checkmark$ \\
Transformer + DRoPE & 195.51 ms & 5.08 GB & 46.03 ms & $\color{red}\times$ \\
\ourmodel-3M & 264.34 ms & 6.35 GB & 58.01 ms & $\color{green}\checkmark$ \\
\bottomrule
\end{tabular}
\caption{Peak memory and latency comparisons of the baselines and \ourmodel.
All methods in the same group use the same context length.
The training latency times the combined forward and backward passes.
The inference latency are averaged over the entire 8s rollout.
Both training and inference use batch size = 8.
All entries are evaluated a single A5000 with fixed random seed.
}
\label{table:memory_footprint}
\end{table}

We use FLOPs as the main compute measurement since FLOPs are hardware agnostic, but provide additional memory footprint and latency comparisons in Table \ref{table:memory_footprint}.
Compared to a fully invariant baseline (Transformer + RPE), \ourmodel is faster and more memory efficient.
For a fixed batch size, \ourmodel has lower throughput than the non-\SE2 equivariant baselines.
However, we demonstrate that \ourmodel achieves the superior Pareto front in Figure \ref{fig:scaling} of the main manuscript by fixing the resources, training each model on the same device and maxing out the batch size for each model, which we believe makes a fair and sufficient comparison of Pareto fronts.

\end{document}